\documentclass[final]{siamltex}
\usepackage{amsmath}
\usepackage{amssymb}
\usepackage{graphicx}
\usepackage{subfigure}
\usepackage{multicol}
\usepackage{multirow}
\usepackage{color}
\usepackage[boxed]{algorithm2e} 

\newcommand{\omitme}[1]{}

\def\R{\mathbb{R}}
\def\1{\mathbf{1}}
\def\diag{\mbox{\boldmath{diag}}}
\DeclareMathOperator*{\argmin}{arg\,min}

\newcommand{\change}[1]{\color{black}{#1}}

\title{Combinatorial Continuous Maximum Flow}

\author{Camille Couprie \thanks{Universit\'e Paris-Est, Laboratoire
    d'Informatique Gaspard-Monge, Equipe A3SI, ESIEE Paris (93160
    Noisy-le-Grand, France) ({\tt c.couprie@esiee.fr}, {\tt
       h.talbot@esiee.fr}, {\tt
      l.najman@esiee.fr} ).}
        \and Leo Grady \thanks{Siemens Corporate Research, Department
          of Imaging and Visualization, Princeton, N.J. 08540. USA) ({\tt leo.grady@siemens.com}).}
        \and Hugues Talbot~$^*$
        \and Laurent Najman~$^*$ }

\begin{document}

\maketitle

\begin{abstract}
  Maximum flow (and minimum cut) algorithms have had a strong impact
  on computer vision. In particular, graph cuts algorithms provide a
  mechanism for the discrete optimization of an energy functional
  which has been used in a variety of applications such as image
  segmentation, stereo, image stitching and texture
  synthesis. Algorithms based on the classical formulation of max-flow
  defined on a graph are known to exhibit metrication artefacts in the
  solution.  Therefore, a recent trend has been to instead employ a
  spatially \emph{continuous} maximum flow (or the dual min-cut
  problem) in these same applications to produce solutions with no
  metrication errors. However, known fast continuous max-flow
  algorithms have no stopping criteria or have not been proved to
  converge.  In this work, we revisit the continuous max-flow problem
  and show that the analogous discrete formulation is \emph{different}
  from the classical max-flow problem. We then apply an appropriate
  combinatorial optimization technique to this combinatorial
  continuous max-flow (CCMF) problem to find a null-divergence
  solution that exhibits no metrication artefacts and may be solved
  exactly by a fast, efficient algorithm with provable
  convergence. Finally, by exhibiting the dual problem of our
    CCMF formulation, we clarify the fact, already proved by Nozawa in
    the continuous setting, that the max-flow and the total variation
    problems are not always equivalent.

\end{abstract}

\section{Introduction}

Optimization methods have been used to address a wide variety of
problems in computer vision.  The early optimization approaches were
formulated in terms of active contours and surfaces~\cite{snakes1987}
and then later level sets~\cite{sethian_book}. These formulations were
used to optimize energies of varied sophistication (e.g., using
regional, texture, motion or contour terms~\cite{paragios_thesis}) but
generally converged to a local minimum, generating results that were
sensitive to initial conditions and noise levels. Consequently, more
recent focus has been on energy formulations (and optimization
algorithms) for which a global optimum can be found.

Such energy formulations typically include a term which minimizes the boundary length
(or surface area in 3D) of a region or the total variation of a scalar field in
addition to a data term and/or hard constraints.  In this paper, we focus on image
segmentation as the example application on which to base our exposition. Indeed,
segmentation has played a prominent (and early) role in the development of the
use of global optimization methods in computer vision, and often forms the basis
of many other
applications~\cite{Boykov_Y_2001_tpami_fas_aemvgc,Ishikawa_H_2003_tpami_exact_omrfcp,
  darbon-sigelle-JMIV2006b}.

\subsection*{Graph-based approaches} 
The max-flow/min-cut problem on a graph is a classical problem in
graph theory, for which the earliest solution algorithm goes back to
Ford and Fulkerson~\cite{FordFulk}. Initial methods for global
optimization of the boundary length of a region formulated the energy
on a graph and relied on max-flow/min-cut methods for
solution~\cite{Boykov2001Interactive,energyKolmo}.  It was soon
realized that these methods introduced a metrication error. 
  Metrication errors are clearly visible when gradient information is
  weak or missing, for example in the case of medical image
  segmentation or in some materials science applications like electron
  tomography, where weak edges are unavoidable features of the imaging
  modality. Metrication errors are also far more obvious in 3D than in 2D and
  increasingly so, as the dimensionality increases (see for example
  the 3D-lungs example in \cite{hugues}).  Furthermore, metrication
  artifacts are even more present in other applications such as
  surface reconstruction and rendering, where the artifacts are a lot
  worse than in image segmentation. Various solutions for
metrication errors were proposed.  One solution involved the use of a
highly connected lattice with a specialty edge
weighting~\cite{Boykov03minsurf}, but the large number of graph edges
required to implement this solution could cause memory concerns when
implemented for large 2D or 3D images~\cite{hugues}.

\subsection*{Continuous approaches}
To avoid the metrication artefacts without increasing memory
usage, one trend in recent years has been to pursue spatially
\emph{continuous} formulations of the boundary length and the related
problem of total variation
\cite{chambolle2004,hugues,nikolova2006}. Historically, a continuous
max-flow (and dual min-cut problem) was formulated by Iri
  \cite{iri} and then Strang~\cite{strang1983}. Strang's continuous
formulation provided an example of a \emph{continuization} (as opposed
to \emph{discretization}) of a classically discrete problem, but was
not associated to any algorithm. Work by Appleton and
Talbot~\cite{hugues} provided the first PDE-based algorithm to find
Strang's continuous max-flows and therefore optimal min-cuts.  This
same algorithm was also derived by Unger~{\it{et al.}}  from the
standpoint of minimizing continuous total
variation~\cite{unger2008tvseg}.  Adapted to image segmentation, this
algorithm is shown to be equivalent~\cite{ungerReport} to the
Appleton-Talbot algorithm and has been demonstrated to be fast when
implemented on massively parallel architectures. Unfortunately, this
algorithm has no stopping criteria and has not been proved to
converge. Works by Pock~{\it et al}~\cite{pock2008ECCV}, Zach~{\it et
  al}~\cite{zach2008} and Chambolle~{\it et al}~\cite{Chambolle2008}
present different algorithms for optimizing comparable energies for
solving multilabel problems, but again, those algorithms are not
proved to converge.  Some other works have been presenting provably
converging algorithms, but with relatively slow convergence speed. For
example, Pock and Chambolle introduce in ~\cite{ChambollePock2010} a
general saddle point algorithm that may be used in various
applications, but needs half an hour to segment a $350 \times 450$
image on a CPU and still more than a dozen seconds on a GPU
{\change(Details are given in the experiments section)}.

\subsection*{Links and differences with total variation minimization}
G. Strang \cite{strang2008} has shown the continuous max-flow problem
for the $l_2$ norm to be the dual of the total variation (TV) minimization
problem under some assumptions. The problem of total
variations, introduced successively in computer vision by
  Shulman and Herv\'e \cite{shulman1989opticalTV} and later Rudin,
  Osher and Fatemi~\cite{rudin1992total} as a regularizing criterion
  for optical flow computation and then image denoising, has been
shown to be quite efficient for smoothing images without affecting
contours. Moreover, a major advantage of TV is that it is a convex
problem.  Thus, a straightforward algorithm such as gradient descent
may be applied to find a minimum solution. However, there is a need
for faster methods, and significant progress have been achieved
recently with primal-dual approaches \cite{chambolle2004}, Nesterov's
algorithm~\cite{Nesterov2007gradient} and
Split-Bregman/Douglas-Rachford methods
\cite{GoldsteinOsher,NonLocalBregman}. Most methods
  minimizing TV focus on image filtering as an application, and even
  if those methods are remarkably fast in denoising applications,
  segmentation problems require a lot more iterations for those
  algorithms to converge. As stated previously, continuous max-flow
is dual with total variation {\it in the continuous setting under restrictive regularity conditions ~\cite{nozawa1990max}}. In fact,
  Nozawa~\cite{nozawa1994duality} showed that there is a duality gap
  between weighted TV and weighted max-flow under some conditions in
  the continuous domain. Thus, it is important to note that weighted
  TV problems are not equivalent to the weighted maximum flow problem
  studied in this paper. Furthermore, many works assume that the
continuous-domain duality holds algorithmically, but we show later in
this paper that at least in the combinatorial case this is not true.

The previous works for solving the max-flow problem illustrate two
  difficulties with continuous-based formulation, that are (1) the
discretization step, which is necessary for deriving algorithms, but
may break continuous-domain properties; and (2) the convergence, both
of the underlying continuous formulation, and the associated
algorithm, which itself depend on the discretization. It is very well
known that even moderately complex systems of PDEs may not converge,
and even existence of solutions is sometimes not
obvious~\cite{NavierStokes}. Even when existence and convergence
proofs both exist, sometimes algorithmic convergence may be slow in
practice.
For these reasons, combinatorial approaches to the maximal flow and related problems are beginning to emerge.

%

\subsection*{Combinatorial approaches}
Discrete calculus~\cite{hirani,grady2010discrete} has been used in
recent years to produce a combinatorial reformulation of continuous
problems onto a graph in such a manner that the solution behaves
analogously to the continuous formulation (e.g.,
\cite{elmoataz2008nonlocal,grady2009piecewise}).

In particular, Lippert presented in~\cite{Lippert2006} a combinatorial
formulation of an isotropic continuous max-flow problem on a planar
lattice, making it possible to obtain a provably optimal
solution. However, in Lippert's work, parameterization of the capacity
constraint is tightly coupled to the 4-connected squared grid, and the
generalization to higher dimensions seems non-trivial. Furthermore it
involves the multiplication of the number of capacity constraints by
the degree of each node, thus increasing the dimension of the
problem. This formulation did not lead to a fast
  algorithm. The author compared several general solvers, quoting
  an hour as their best time for computing a solution on a $300 \times 300$ lattice.

\omitme{General convex solvers were used to compute the solution of this formulation, but they run very slowly.}

\subsection*{Motivation and contributions}
In this paper, we pursue a combinatorial reformulation of the
continuous max-flow problem initially formulated by Strang, which we
term \emph{combinatorial continuous maximum flow} (CCMF). Viewing our
contribution differently, we adopt a discretization of continuous
max-flow as the primary problem of interest, and then we apply fast
combinatorial optimization techniques to solve the discretized
version.

Our reformulation of the continuous max-flow problem produces a
(divergence-free) flow problem defined on an arbitrary
graph. Strikingly, CCMF are not equivalent to the discretization of
continuous max-flows produced by Appleton and Talbot or that of
Lippert (if for no other reason than the fact that CCMF is defined on
an arbitrary graph). Moreover, \emph{CCMF is not equivalent to the
  classical max-flow on a graph}. In particular, we will see that the
difference lies in the fact that capacity constraints in classical
max-flow restrict flows along graph edges while the CCMF capacity
constraints restrict the total flow passing through the nodes.

 The CCMF problem is convex. We deduce an expression of the dual
problem, which allows us to employ a primal-dual interior point method
to solve it, such as interior point methods have been used
  for graph cuts~\cite{BhusnurmathPAMI08} and second order cone
  problems in general~\cite{Goldfarb2005cone}. The CCMF problem has
several desirable properties, including:
\begin{enumerate}
\item the solution to CCMF on a 4-connected lattice avoids the
  metrication errors. Therefore, the gridding error problems may be
  solved without the additional memory required to process classical
  max-flow on a highly-connected graph;
\item in contrast to continuous max-flow algorithms of Appleton-Talbot (AT-CMF)
  and equivalent, the solution to the CCMF problem can be computed
  with guaranteed convergence in practical time;
\item the CCMF problem is formulated on an arbitrary graph, which
  enables the algorithm to incorporate nonlocal edges
  \cite{nlmeans,elmoataz2008nonlocal,grady2004:faster} or to apply it
  to arbitrary clustering problems defined on a graph; and
\item the algorithm for solving the CCMF problem is fast, easy to
  implement, compatible with multi-resolution grids and is
  straightforward to parallelize.
\end{enumerate}

Our computation of the CCMF dual further reveals that duality between
total variation minimization and maximum flow does not hold for CCMF
and combinatorial total variation (CTV). The comparison between those
two combinatorial problems is motivated by several interests:
\begin{enumerate}
\item for clarification: in the continuous domain, the duality between TV and MF
  holds under some regularity constraints. In the discrete anisotropic domain
  (i.e. Graph Cuts), this duality always holds. However, In the isotropic weighted
  discrete case, i.e. in the CTV, AT-CMF or CCMF cases, we are not aware of
  any discretization such that the duality holds.
  It is not obvious to realize that the duality
  does not hold in the discrete setting even though it may in the continuous
  case. We present clearly the differences between the two problems,
  theoretically, and in term of results;
\item to expose links between CCMF and CTV: We prove that the weak
  duality holds between the two problems; and
\item for efficiency: in image segmentation, the fastest known
  algorithms to optimize CTV are significantly slower than CCMF. Therefore,
  there is reason to believe that CCMF may be used to efficiently
  optimize energies for which TV has been previously shown to be
  useful (e.g., image denoising).
\end{enumerate} 

In the next section, we review the formulation of continuous max-flow,
derive the CCMF formulation and its dual and then provide details of
the fast and provably convergent algorithm used to solve the new CCMF problem.

\section{Method} 

Our combinatorial reformulation of the continuous maximum flow problem
leads to a formulation on a graph which is \emph{different} from the
classical max-flow algorithm.  Before proceeding to our formulation,
we review the continuous max-flow and the previous usage of this
formulation in computer vision. 

\subsection{The continuous max-flow (CMF) problem}

First introduced by Iri \cite{iri}, Strang presents in
\cite{strang1983} an expression of the continuous maximum flow problem
%

\begin{align}
\begin{gathered}
\max {F_{st}},\;\;\;\;\;\;\\
\text{s.t.}\;\;\;  \nabla \cdot \overrightarrow{F} = 0,\\
||\overrightarrow{F}|| \leq g.
\label{eq:cont_MF}
 \end{gathered}
 \end{align}

 Here we denote by $F_{st}$ the total amount of flow going from the
 source location $s$ to the sink location $t$, $\overrightarrow{F}$ is
 the flow, and $g$ is a scalar field representing the local metric
 distortion. The solution to this problem is the exact solution of the
 geodesic active contour (GAC) (or surface)
 formulation~\cite{caselles1997geodesic}.  In order to solve the
 problem \eqref{eq:cont_MF}, the Appleton-Talbot algorithm (AT-CMF)
 \cite{hugues} solves the following partial differential equation
 system:

\begin{eqnarray}
\nonumber \frac{\partial P}{\partial \tau} = -\nabla \cdot
\overrightarrow{F},\\
\label{eq:cmf}
\frac{\partial \overrightarrow{F}}{\partial \tau} = -\nabla P,\\
\nonumber \mbox{s. t. } {||\overrightarrow{F}||} \leq g. 
\label{AT}
\end{eqnarray}
 
Here $P$ is a potential field similar to the excess value in the Push-Relabel
maximum flow algorithm~\cite{tarjan}. AT-CMF is effectively a simple continuous
computational fluid dynamics (CFD) simulation with non-linear constraints. It
uses a forward finite-difference discretization of the above PDE system subject
to a Courant-Friedrich-Levy (CFL) condition, also seen in early level-sets
methods. At convergence, the potential function $P$ approximates an indicator
function, with $0$ values for the background labels, and $1$ for the foreground, 
 and the flow $\overrightarrow{F}$ has zero-divergence almost everywhere.
However, there is no guarantee of convergence for this algorithm and,
in practice, many thousands of iterations can be necessary to achieve
a binary $P$, which can be very slow.

Although Appleton-Talbot is a continuous approach, applying this
algorithm to image processing involves a discretization step.
The capacity constraint $||\overrightarrow{F}|| \leq g$ is interpreted
as $\max{{F_{x(i)}}^2} + \max{{F_{y(i)}}^2} \leq g_i^2$, with
$F_{x(i)}$ the outgoing flow of edges linked to node $i$ along the $x$
axis, and $F_{y(i)}$ the outgoing flow linked to node $i$ along the
$y$ axis. We notice that the weights are associated with point
locations (pixels), which will correspond later to a node-weighted
graph.

In the next section, we use the operators of discrete calculus to
reformulate the continuous max-flow problem on a graph and show the
surprising result that the graph formulation of the continuous
max-flow leads to a different problem from the classical max-flow
problem on a graph.

\subsection{A discrete calculus formulation}

Before establishing the discrete calculus formulation of the
continuous max-flow problem, we specify our notation. A graph consists
of a pair $G = (V,E)$ with vertices $v \in V$ and edges $e \in E
\subseteq V \times V$. A transport graph $G(V, E)$ comprises two
additional nodes, named a source $s$ and a sink $t$, and additional
edges linking some nodes to the sink and some to the source. Including
the source and the sink nodes, the cardinalities of $G$ are given by
$n = |V|$ and $m = |E|$. An edge, $e$, spanning two vertices, $v_i$
and $v_j$, is denoted by $e_{ij}$. In this paper we deal with weighted
graphs that include weights on both the edges and nodes. An edge
weight is a value assigned to each edge which may be viewed as a
capacity in the context of a maximum flow problem.  The weight of an
edge, $e_{ij}$, is denoted by $\tilde{g}_{ij}$. In this work, we
assume that $\tilde{g}_{ij} \in \R$ and $\tilde{g}_{ij} > 0$, 
  and use $\tilde{g}$ to denote the vector of $\R^m$ containing the
  $\tilde{g}_{ij}$ for every edge $e_{ij}$ of $G$. In
addition to edge weights, we may also assign weights to nodes. The
weight of node $v_i$ is denoted by $g_i$. In this work, we also assume
that $g_{i} \in \R$ and $g_{i} > 0$.  We use $g$ to denote the vector of $\R^n$ containing the
  $g_{ij}$ for every edge $e_{ij}$ of $G$. We define a flow through edge
$e_{ij}$ as $F_{ij}$ where $F_{ij} \in \R$ and use the vector $ F \in \R^m$ to
denote the flows through all edges in the graph .  Each edge is assumed
to be oriented, such that a positive flow on edge $e_{ij}$ indicates
the direction of flow from $v_i$ to $v_j$, while a negative flow
indicates the direction of flow from $v_j$ to $v_i$.

The incidence matrix is a key operator for defining a combinatorial
formulation of the continuous max-flow problem.  Specifically, the
incidence matrix $ A \in \R^{m \times n}$ is known to define the discrete calculus analogue
of the gradient, while $A^T$ is known to define the discrete calculus
analogue of the divergence (see \cite{grady2010discrete} and the
references therein). The incidence matrix maps functions on nodes (a
scalar field) to functions on edges (a vector field) and may be
defined as
\begin{equation}
A_{e_{ij} v_k}= \begin{cases}
                +1& \text{if $i=k$},\\
                -1& \text{if $j=k$},\\
                0& \text{otherwise},
        \end{cases} 
\label{eq:incidence}
\end{equation}
for every vertex $v_k$ and edge $e_{ij}$. In our formulation of
continuous max-flow, we use the expression $|A|$ to denote the matrix
formed by taking the absolute value of each entry individually.

Given these definitions, we now produce a discrete (combinatorial)
version of the continuous max-flow of \eqref{eq:cont_MF} on a 
transport graph. As in
\cite{grady2010discrete,elmoataz2008nonlocal,grady2009piecewise}, the
continuous vector field indicating flows may be represented by a
vector on the edge set, $F$.  Additionally, the combinatorial
divergence operator allows us to write the first constraint in
\eqref{eq:cont_MF} as $A^T F = 0$.  The second constraint in
\eqref{eq:cont_MF} involves the comparison of the point-wise
norm of a vector field with a scalar field.  Therefore, we can follow
\cite{grady2010discrete,elmoataz2008nonlocal,grady2009piecewise} to
define the point-wise $\ell_2$ norm of the flow field $F$ as
$\sqrt{|A^T| F^2}$. In our notations here, as in the rest of
  the paper, $F^2 = F \cdot F$ denotes a element-wise product,
  ``$\cdot$'' denoting the Hadamard (element-wise) product between the
  vectors, and the square root of a vector is also here and in the rest
  of the paper the vector composed of the square roots of every elements.

Putting these pieces together, we obtain
\begin{align}
\begin{gathered}
\max {F_{st}},\;\;\;\;\;\;\\
\text{s.t.}\;\;\;  A^T F = 0,\\
 |A^T| F^2 \leq g^2.
\label{f_new}
 \end{gathered}
 \end{align}
Compare this formulation to the classical max-flow problem on a graph,
given in our notation as
\begin{align}
\begin{gathered}
\max {F_{st}},\;\;\;\;\;\;\\
\text{s.t.}\;\;\;  A^T F = 0,\\
 |F| \leq \tilde{g}.
\label{f_old}
 \end{gathered}
 \end{align}

By comparing these formulations of the traditional max-flow with our
combinatorial formulation of the continuous max-flow, it is apparent
that the key difference between the classical formulation and our
combinatorial continuous max-flow is in the capacity constraints. In
both formulations, the flow is defined through edges, but in the
classical case the capacity constraint restricts flow through an edge
while the CCMF formulation restricts the amount of flow passing
through a node by taking in account the neighboring flow
  values. This contrast-weighting applied to nodes (pixels) has been
a feature of several algorithms with a continuous formulation,
including geodesic active contours \cite{caselles1997geodesic}, CMF
\cite{hugues}, TVSeg \cite{unger2008tvseg}, and
  \cite{ChambollePock2010}. On the other hand, the problem
defined in \eqref{f_new} is defined \emph{on an arbitrary graph} in
which contrast weights (capacities) are almost always defined on
\emph{edges}. The node-weighted capacities fit Strang's formulation of
a scalar field of constraints in the continuous max-flow formulation
and therefore our graph formulation of Strang's formulation carries
over these same capacities. The most important point here is
  not having expressed the capacity on the nodes of the graph but
  rather how the flow is enforced to be bounded by the metric in each
  node. Bounding the norm of neighboring flow values in each node
  simulates a closer behavior of the continuous setting than if no
  contribution of neighboring flow was made as in the standard
  max-flow problem. Figure \ref{diff_discretization} illustrates the
relationship of the edge capacity constraints in the classical
max-flow problem to the node capacity constraints in the CCMF
formulation.  The null-divergence constraint is also essential
  in our formulation since it permits us to obtain constant partitions
  almost everywhere in the dual solution introduced in the next
  section, in particular binary ones.
\begin{figure}[ht]
\begin{center}
\begin{frame}{\includegraphics[width=0.85 \linewidth]{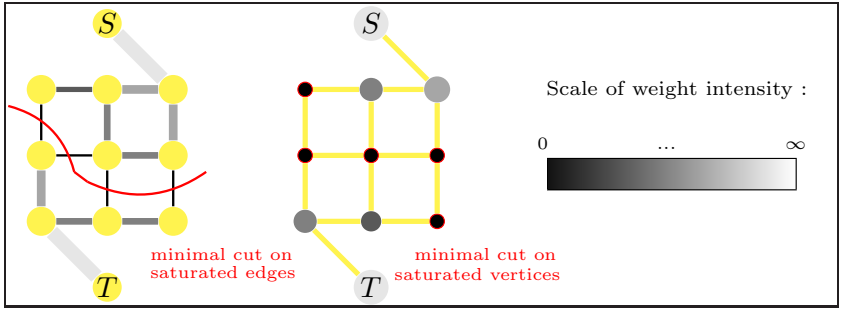}}\end{frame}
\end{center}
\caption{The difference between classical max-flow on a graph with the
  combinatorial continuous max-flow (CCMF) on a graph is that classical
  max-flow uses edge-weighted capacities while CCMF uses node-weighted
  capacities. This difference is manifest in the different solutions
  obtained for both algorithms and the algorithms required to find a
  solution. Specifically, the solution to the CCMF problem on a
  lattice does not exhibit metrication bias.}
\label{diff_discretization}
\end{figure}

In the context of image segmentation, the vector $g$ varies
inversely to the image gradient. We propose to use, as
in~\cite{unger2008tvseg},
\begin{equation}
g =  \exp(-\beta ||\nabla I||_2),
\label{eq_exp_metric}
\end{equation}
where $I$ indicates the image intensity. For simplicity, this
weighting function is defined for greyscale images, but $g$
could be used to penalize changes in other relevant image quantities,
such as color or texture. Before addressing the solution of the CCMF
problem, we consider the dual of the CCMF problem and, in particular,
the sense in which it represents a minimum cut. Since the cut weights
are present on the nodes rather than the edges, we must expect the
minimum cut formulation to be different from the classical minimum cut
on a graph.

\subsection{The CCMF dual problem}

Classically, the max-flow problem is dual to the min-cut problem,
allowing a natural geometric interpretation of the objective function.
In order to provide the same interpretation, we now consider the dual
problem to the CCMF and show that we optimize a node-weighted minimum
cut. In the following proposition, we denote the element-wise quotient
  of two vectors $u=[u_1, \ldots, u_k]$ and $v=[v_1, \ldots, v_k]$ by $
  u\cdot/v = [u_1/v_1, \ldots, u_k/v_k]$. We also denote $\1^n$ a unit
  vector [1, \ldots, 1] of size $n$, and recall that the square
  exponent $v^2$ of a vector $v$ represents the resulting vector of
  the element-wise mulitiplication $v\cdot v$.

\begin{proposition}
In a transport graph $G$ with $m$ edges, $n$ nodes, we
define a vector $c$ of $\R^m$, composed of zeros except for the
element corresponding to the source/sink edge which is $1$.  Let
$\lambda$ and $\nu$ be two vectors of $\R^n$.  The CCMF problem
\begin{eqnarray}
\label{problem}  
      \nonumber   \max_{ {F \in \R^m}} & c^T F,\\
       \mbox{s. t.} & A^T F = 0,\\
      \nonumber ~ & |A^T| F^2 \leq g^2.
\end{eqnarray}
has for dual 
\begin{eqnarray}
\label{unsimplified_dual}
 \min_{{\lambda\in\R^n,~\nu\in\R^n}} & ~\lambda^T g^2 + \frac{1}{4} {\bigg( \1^n \cdot/ (|A| \lambda)\bigg)^T \bigg((c+A\nu)^2\bigg)},\\
\nonumber \mbox{s. t.} &\lambda \geq 0,
\end{eqnarray}
equivalently written in \eqref{weighted_dual}, and the optimal solution $(F^*, \lambda^*, \nu^*)$ verifies 
\begin{equation}
\label{dual_lambda_nu}
\max_F{c^TF} =  c^T F^* = 2{\lambda^*}^T g^2,
\end{equation}
and the $n$ following equalities
\begin{equation}
 \lambda^* \cdot |A^T|\bigg((c+A\nu^*) \cdot/ (|A|\lambda^* )\bigg)^2 = 4\lambda^* \cdot g^2.
\end{equation}

\end{proposition}

\begin{proof}
The Lagrangian of \eqref{problem} is 
\begin{equation*}
L(F,\lambda,\nu) =(c^T + \nu^T A^T) F + \lambda^T |A^T| F^2 - \lambda^T g^2,
\end{equation*}
with  $\lambda \in \R^n$ and $\nu \in \R^n$ two Lagrange multipliers.
We have to find $F$ in the dual function is such as $\nabla_F L(F,\lambda,\nu) = 0 $.
\begin{equation}
\nabla_F L(F,\lambda,\nu) = c+A\nu + 2|A|\lambda \cdot F = 0
 ~\Leftrightarrow~ F = (c+A\nu)\cdot/(-2|A|\lambda). 
\label{r_dual}
\end{equation}
Substituting $F$ in the Lagrangian function \omitme{\eqref{dual_function}}, we obtain 
  \begin{equation*}
L(\lambda,\nu) = (c^T+\nu^T A^T) \bigg( (c+A\nu)\cdot/(-2|A|\lambda)\bigg)  + \lambda^T|A^T| {\bigg( (c+A\nu)\cdot/(-2|A|\lambda)\bigg)}^2 - \lambda^T g^2.
  \end{equation*} 
The Lagrangian may be simplified as:
   \begin{equation*}
  L(\lambda, \nu)= - {\bigg(\1^n \cdot/ (4|A|\lambda)\bigg)^T\bigg((c+A\nu)^2\bigg)} - \lambda^T g^2.
  \end{equation*} 
The dual of \eqref{problem} is also
\begin{eqnarray}
\label{unsimplified_dual2}
\min_{\lambda, \nu} & ~\lambda^T g^2 + {\frac{1}{4} {\bigg( \1^n \cdot/ (|A| \lambda)\bigg)^T \bigg((c+A\nu)^2\bigg)}}, \\
\nonumber \mbox{s. t. } &\lambda \geq 0.
\end{eqnarray}
\omitme{
At convergence, we notice that the solution $(\lambda, \nu)$ of
\eqref{unsimplified_dual} satisfies
 \begin{equation*}
\max_F ~ c^T F = \min_{\lambda, \nu} ~ 2\lambda^T g^2,
\end{equation*}
\begin{equation*} 
  \lambda \cdot |A^T|\left( (c+A\nu) \cdot/ (|A|\lambda)  \right)^2 = 4\lambda \cdot g^2. 
\end{equation*}
In effect, the KKT optimality conditions give 
\begin{equation}
c+A\nu+2|A|\lambda \cdot F = 0,
\end{equation}
}
Furthermore, the KKT optimality conditions give 
\begin{equation}
\label{r_cent}
\lambda^* \cdot\left(|A^T|{F^*}^2-g^2\right)=0.
\end{equation}
 Using the expression of $F^*$ from equation \eqref{r_dual} \omitme{gives $ F = \frac{c + A \nu}{ -2|A|\lambda } $}, and substituting it in \eqref{r_cent}, we obtain
\begin{equation*}
 \lambda^* \cdot
 |A^T|{\bigg((c+A\nu^*)\cdot/(|A|\lambda^*)\bigg)^2} = 4\lambda^* \cdot g^2. 
\end{equation*}
Taking the sum of right hand side and left hand side, 
\begin{equation}
\label{eq8}
 (c+A\nu^*)^2 \cdot/ (|A|\lambda^*)   = 4{\lambda^*}^T g^2. 
\end{equation}
Finally, if we substitute \eqref{eq8} in the dual expression
\eqref{unsimplified_dual}, we have
\begin{equation*}
\max_F ~ c^T F =  c^T F^* = 2{\lambda^*}^T g^2.
\end{equation*}
\end{proof}

The expression of the CCMF dual may be written in summation form as
\begin{align}
\begin{small}
\begin{gathered}
 \min_{\lambda, \nu} \sum_{v_i \in V}{\overbrace{\lambda_i
     g_i^2}^{\mbox{weighted cut}}} + \overbrace{\frac{1}{4}
   \sum_{e_{ij}\in
     E\setminus\{s,t\}}{\frac{(\nu_i-\nu_j)^2}{\lambda_i+\lambda_j}}}^{\mbox{smoothness
 term}} + \overbrace{\frac{1}{4}
   \frac{(\nu_s-\nu_t-1)^2}{\lambda_s+\lambda_t}}^{\mbox{source/sink enforcement}} \\
 \text{s. t. } \lambda_i \geq 0 ~~~ \forall i \in V.
\label{weighted_dual}
 \end{gathered}
\end{small}
 \end{align}

\noindent{\bf Interpretation: } The optimal value {$\lambda^*$} is a weighted indicator of
the saturated vertices (a vertex $v_i$ is saturated if $|A^T|_i F^2 =
g_i^2$ where $|A^T|_i$ indicates the $i$th row of $|A^T|$):
\begin{equation} 
\lambda^*(v_i) \left\{
    \begin{array}{ll}
        >0 ~~\mbox{  if } {|A^T|}_i F^2 = g(v_i)^2, \\
        =0 ~~\mbox{  otherwise.}\\
    \end{array}
\right.
\label{interpretation}
\end{equation}
The variables $\nu_s$ and $\nu_t$ are not constrained to be set to 0
and 1, only their difference is constrained to be equal to one. Thus their value range
between a constant and a constant plus one. Let us call this constant
$\delta$, and without loss of generality one can consider
 that $\delta=0$. The term $\nu$ is at optimality a weighted
indicator of the source/sink/saturated~vertices partition:
\begin{equation*} 
\nu^*(v_i) = \left\{
    \begin{array}{ll}
        0 + \delta ~~\mbox{  if } v_i \in S, \\
        \mbox{a number between }(0 + \delta) \mbox{ and } (1 + \delta) ~~\mbox{  if } {|A^T|}_i F^2 = g(v_i)^2,\\
        1 + \delta ~~\mbox{  if } v_i \in T. \\
    \end{array}
\right.
\end{equation*}

The expression \eqref{dual_lambda_nu} of the CCMF dual shows that the
problem is equivalent to finding a minimum weighted cut defined on the
nodes.

Finally, the ``weighted cut'' is recovered in
\eqref{weighted_dual}, and the ``smoothness term'' is compatible with large
variations of $\nu$ at the boundary of objects because of a large
denominator ($\lambda$) in the contour area.
An illustration of optimal $\lambda$ and $\nu$ on an image is shown on Fig.~\ref{peppers}.

\begin{figure}[tb]
\centering{
    \subfigure[]
    {\label{pepper-a}\includegraphics[width=0.24\textwidth]{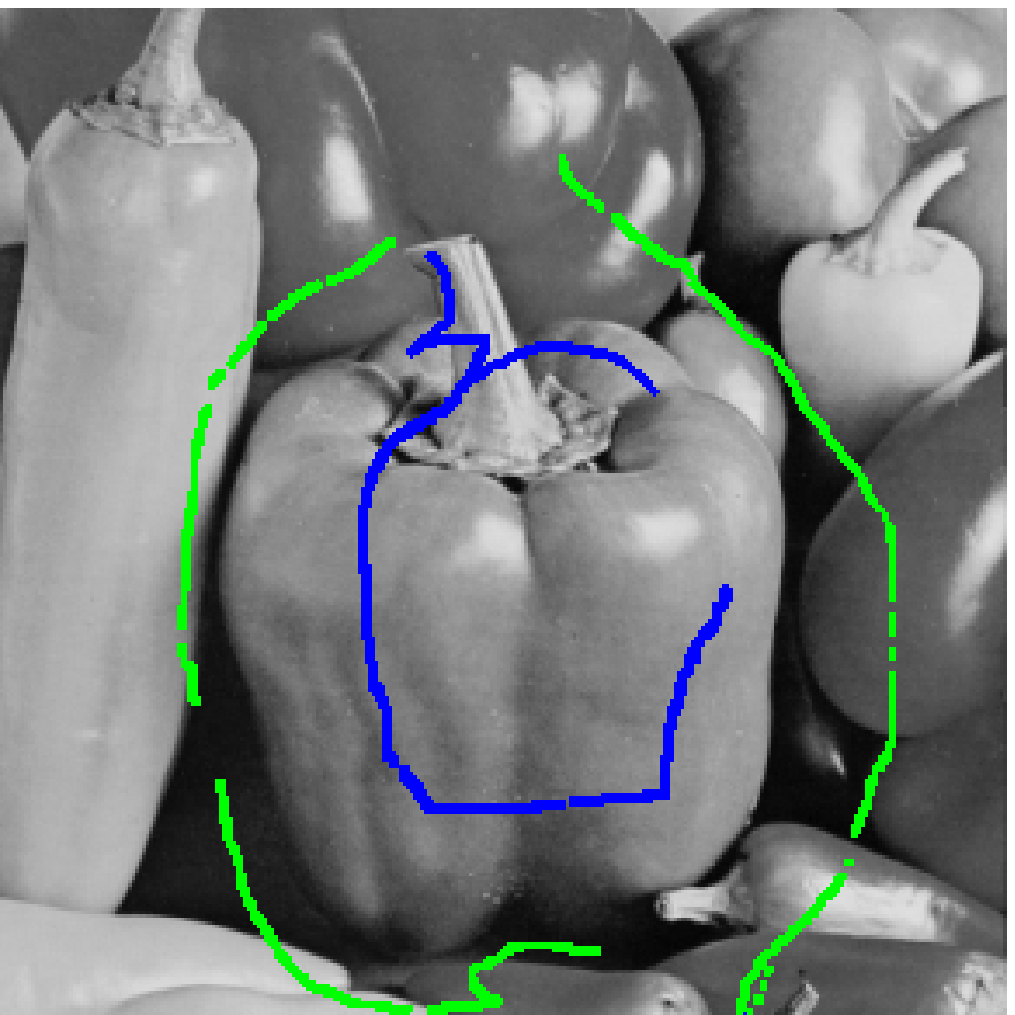}}~~
   \subfigure[$\lambda$]
   {\label{pepper-b}\begin{frame}{\includegraphics[width=0.24\textwidth]{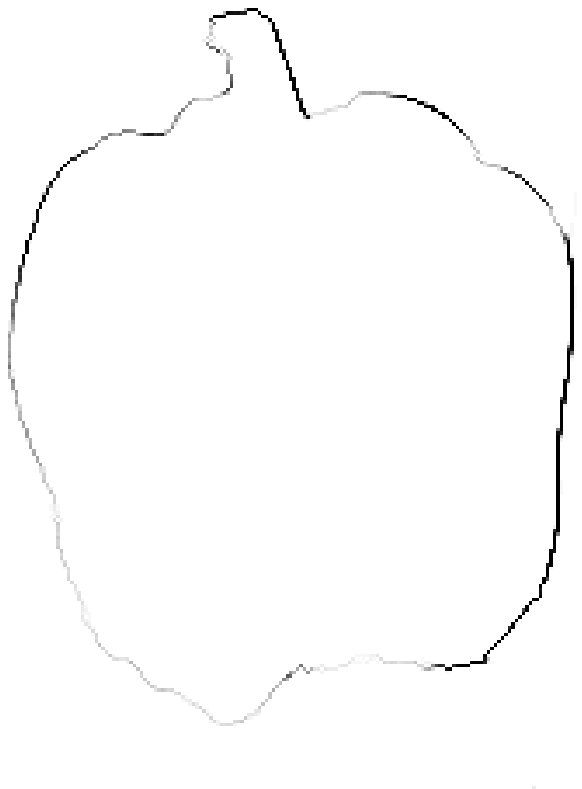}}\end{frame}}~~
    \subfigure[$\nu$]
    {\label{pepper-c}\begin{frame}{\includegraphics[width=0.24\textwidth]{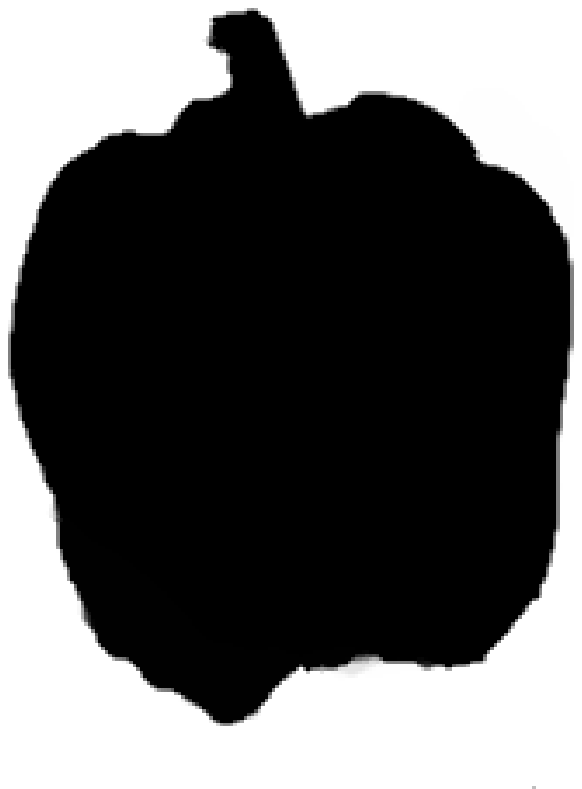}}\end{frame}}~~
  \subfigure[Threshold $\nu$ at $.5$]
            {\label{pepper-d}\includegraphics[width=0.24\textwidth]{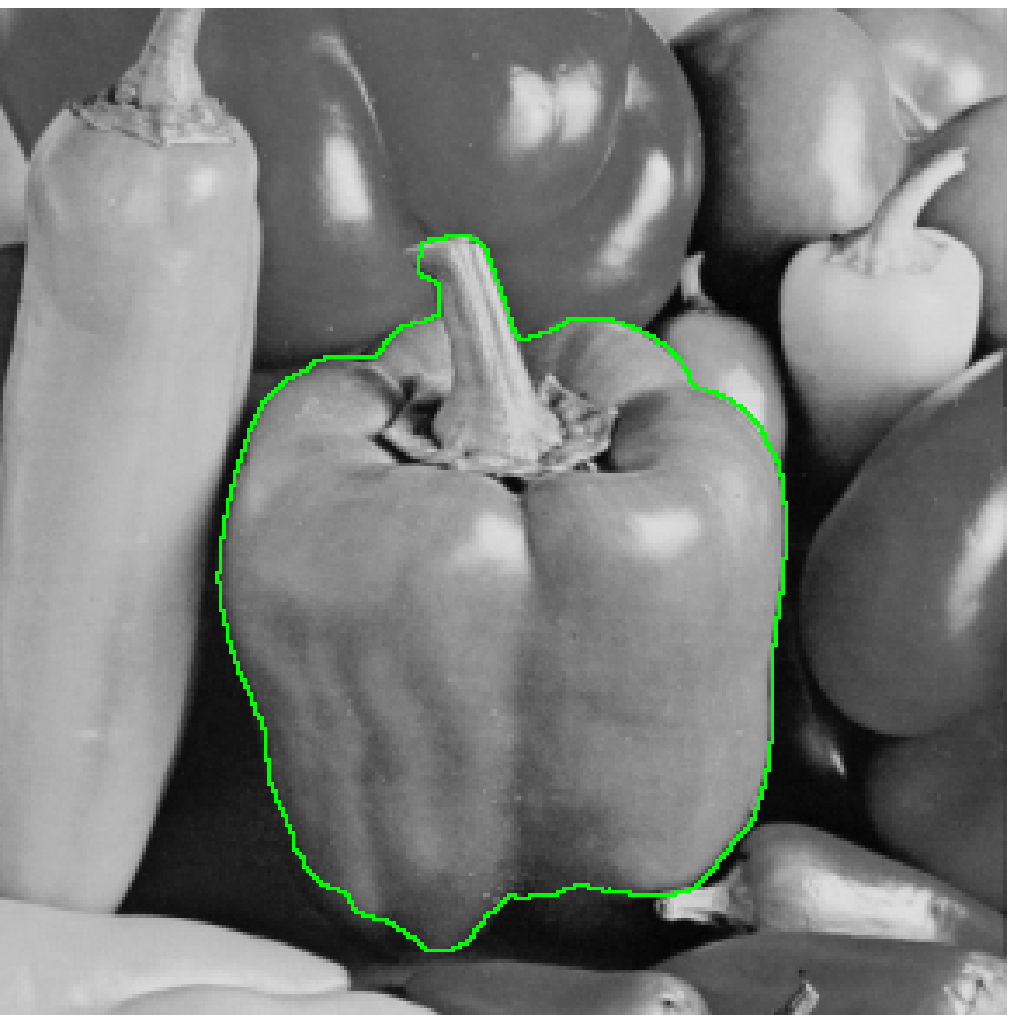}}
}
  \caption{The dual problem to CCMF is a node-weighted minimum cut in
    which the variable $\lambda$ is a weighted indicator vector
    labeling boundary nodes and the variable $\nu$ is a nearly binary
    vector indicating the source/sink regions. As a result, the
      contours of $\nu$ are slightly blurry. This is due to the
      equilibrium effect between the two dual variables. In practice,
      as $\lambda$ is nonzero only in presence of a contour, $\nu$ is
      binary almost everywhere, except on a very thin line.}
\label{peppers}
\end{figure}

\subsection{Solving the Combinatorial Continuous Max-Flow problem}
\label{secPDIP}

When considering how to optimize the CCMF problem \eqref{f_new}, the first key
observation is that the constraints bound a convex feasible region.

 The real valued functions $f_i : \R^m \rightarrow \R$ defined for every node $i=1, \ldots, n$ as $f_i(F)={|A^T|_i} F^2 -
g_i^2$ are non-negative quadratic, so convex functions. We define the
vector $f(F)$ of $\R^n$ as $f(F) = [f_1(F), \ldots, f_n(F)]^T$.

 Since the constraints are convex, the
CCMF problem may be solved with a fast primal dual interior point method (see
\cite{convex}), which we now review in the specific context of CCMF.

\vspace{1ex}
\restylealgo{algoruled}
 \begin{algorithm}[H]
\dontprintsemicolon \SetVline
\KwData{ {$F \in \R^m = 0, \lambda \in (\R^*_+)^n, \nu \in \R^n$}, $\mu > 1, \epsilon > 0$.}
\KwResult{$F$ solution to the CCMF problem \eqref{problem} such as $F_{st}$ is maximized under the divergence free
  and capacity constraints, and $\nu, \lambda$ solution to the CCMF dual
  problem \eqref{dual_lambda_nu}.}
\SetVline
\Repeat {$||r_p||_2 \leq \epsilon$, $||r_d||_2 \leq \epsilon$, and $\widehat{\eta} \leq \epsilon$.}
{ 
1.~Compute the surrogate duality gap $\widehat{\eta} =
  -f(F)^T\lambda $ and set {$t = \mu n / \widehat{\eta} $}.\\
2.~Compute the primal-dual search direction $\Delta_y$ such as $M \Delta_y = r$.\\
3.~Determine a step length $s > 0$ and set $y=y+s\Delta_y$. ~~{\small($y =[F, \lambda, \nu]^T$)}\\
}
\caption{Primal dual interior point algorithm}
\label{PDIPalg}
\end{algorithm}
\vspace{1ex}

The primal dual interior point (PDIP) algorithm iteratively computes
the primal $F$ and dual $\lambda$, $\nu$ variables so that the
Karush-Kuhn-Tucker (KKT) optimality conditions are satisfied. This
algorithm solves the CCMF problem \eqref{problem} by applying Newton's
method to a sequence of a slightly modified version of the KKT
conditions. The steps of the PDIP procedure are given in Algorithm
\ref{PDIPalg}. Specifically for CCMF, the system  $M \Delta_y = r$
system may be written
{\footnotesize
\begin{equation}
\left[\begin{matrix}
\omitme{\nabla^2 f_0(F) +} \sum_{i=1}^n{\lambda_i \nabla^2 f_i(F)} & Df(F)^T & A \\
-\diag(\lambda)Df(F) & -\diag(f(F))&0 \\
A^T&0&0
\end{matrix} \right] \left[ \begin{matrix}
\Delta_F\\
\Delta_{\lambda}\\
\Delta_{\nu}
\end{matrix} \right] = - \left[ \begin{matrix}
r_d =   c + Df(F)^T \lambda + A \nu \\
r_c = -\diag(\lambda)f(F)-(1/t)\\
r_p = A^TF\\
\end{matrix}\right] 
\label{general_sys}
\end{equation}
}

\noindent
with $r_{d}$, $r_{c}$, and $r_{p}$ representing the dual, central, and
primal residuals. Additionally, the derivatives are given by 
\begin{equation*}
\nabla f_i(F) = 2{|A^T|}_i\cdot F,~~
Df(F)=   \left[ \begin{matrix}
\nabla f_1(F)^T \\
\vdots  \\
\nabla f_n(F)^T  
\end{matrix} \right],
~~
 \nabla^2 f_i(F) = \diag 
 \left[ \begin{matrix}
2 {|A^T|}_{i1}\\
\vdots\\
2 {|A^T|}_{im}\\
\end{matrix} \right],
\end{equation*}
with ``$\cdot$'' denoting the Hadamard (element-wise) product between the
vectors. 

Consequently, the primary computational burden in the CCMF algorithm
is the linear system resolution required by \eqref{general_sys}.
In the result section we present execution times obtained in our
experiments using a conventional CPU implementation.

Observe that, although this linear system is large, it is
  very sparse and is straightforward to parallelize on a GPU
  \cite{bolz2003sparse,kruger2003GPU,grady2005GPU}, for instance using
  an iterative GPU solver.  If it does not necessarily imply a faster
  solution, the asymptotic complexity of modern iterative and modern
  direct solvers is about the same and, in our experience, there has
  been a strong improvement in the performance of a linear system
  solve when going from a direct solver CPU solution to a conjugate
  gradient solver GPU solution, especially as the systems get larger.

\omitme{Further computational efficiency could be obtained by
using a GPU, which is known to provide a straightforward and highly
efficient implementation of iterative solvers for a sparse linear
system.
}


\section{Comparison between CCMF and existing related approaches}

\subsection{Min cut}
As detailed in Section 2.3, there is a difference between the CCMF and the
classic max flow formulation in the capacity constraint. Therefore,
the dual problems are different. As proved by Ford and
Fulkerson~\cite{FordFulk}, and formalized for example in
\cite{BhusnurmathPAMI08}, the min cut problem writes with our
notations of section 2.3
\begin{equation}
\min_u{\tilde{g}^T|Au+c|}.
\end{equation} 
We note that the $\ell_1$ norm of the gradient of the solution $u$ of
the above min-cut expression turns into an $\ell_2$ norm of the
gradient of the solution $\nu$ in the CCMF dual
\eqref{unsimplified_dual}.
\subsection{Primal Dual Total variations}
Unger {\it et al.}~\cite{ungerReport} proposed to optimize the
  following primal dual TV formulation
\begin{align}
\begin{gathered}
\min_u{\max_{F}\sum_{i,j}{(u_{i+1,j}-u_{i,j})F_{i,j}^1+(u_{i,j+1}-u_{i,j})F^2_{i,j}}
+ \mbox{data fidelity}},\\
\mbox{s.t. } \sqrt{(F_{i,j}^1)^2 + (F_{i,j}^2)^2} \leq g_{i,j}.\;\;\;
\end{gathered}
\label{unger_primal_dual_pb}
\end{align}
We can note that the CCMF flow capacity constraint in a 4-connected
lattice is similar to the constraint in \eqref{unger_primal_dual_pb}
with a slight modification, and would be with finite-element notations 
\begin{equation}
\sqrt{(F_{i-1,j}^1)^2 + (F_{i,j-1}^2)^2 + (F_{i,j}^1)^2 + (F_{i,j}^2)^2} \leq g_{i,j}.
\end{equation}
In contrast to this finite element discretization, the provided
CCMF formulation of continuous max-flow is defined arbitrary graphs.
Furthermore, we note that the optimization procedure employed to solve
\eqref{unger_primal_dual_pb} generalizes Appleton-Talbot's
algorithm.


\subsection{Combinatorial total variations}

We now compare the dual CCMF problem with the Combinatorial Total
Variation (CTV) problem. We note that the CCMF dual is different than
CTV and discuss the weak duality of the two problems.


We recall the continuous expression of total variation
given by Strang
\begin{align}
\begin{gathered}
 \min_u \iint_{\Omega}{~g~||\nabla u||~\mbox{ d}x\mbox{ d}y},\\
\text{s. t.}\;\;\; \int_{\Gamma} uf \mbox{ ds} = 1,\;\;\;
\end{gathered}
\label{eq:TVStrang}
\end{align}
where $\Omega$ represents the image domain and $\Gamma$ the boundary
of $\Omega$. The source and sink membership is represented by $f$,
such that $f(x,y)>0$ if $(x,y)$ belongs to the sink, $f(x,y)<0$ if
$(x,y)$ belongs to the source, and $f=0$ otherwise.  \omitme {Noting
$v=(x,y)$, we have

\begin{align}
\begin{gathered}
 \min_u \int_{\Omega}{~g~||\nabla u||~\mbox{ d}v},\\
\text{s. t.}\;\;\; \int_{\Gamma} uf \mbox{ ds} = 1.\;\;\;
\end{gathered}
\label{eq:TV}
\end{align}
}
Considering a transport graph $G$, and a vector $u$ defined on the nodes, this continuous problem may be written with combinatorial operators 

\begin{equation}
\begin{gathered}
\min_{u\in\R^n}  g^T \sqrt{|A^T|{(Au)}^2}, \\
\mbox{s. t. } \; u_s - u_t = 1,
\end{gathered}
\label{eq:CTV}
\end{equation}
where the square root operator of a vector $v = [v_1, \ldots,
  v_k]$ is an element-wise square root operator $\sqrt{v} =
  [\sqrt{v_1}, \ldots, \sqrt{v_k}]$.
Another way to write the same problem is 
\begin{equation}
\begin{gathered}
\min_u \sum_{v_i\in V}{g_i
{\sqrt{\sum_{e_{ij}\in E}{(u_i-u_j)^2}}}},\\ \mbox{s. t. }
u_s - u_t = 1.
\end{gathered}
\end{equation}
We note that in these equations the capacity $g$ must be defined on
the vertices. Although we describe this energy as ``combinatorial
total variation'', due to its derivation in discrete calculus terms,
it is important to note that this formulation is very similar
to the discretization which has appeared previously in the literature
from Gilboa and Osher \cite{gilboa2007nonlocal}, and Chambolle \cite{chambolle2004} (if we allow
$g=1$ everywhere).

\subsubsection{CCMF and CTV are not dual}

\omitme{
Given the solution $(F,\lambda,\nu)$ of CCMF and its dual, we may
define a \textit{$b$-cut}, or \textit{boundary-cut}, as a partition of $V$ such as 
\begin{equation*} 
\left\{
    \begin{array}{ll}
       v_i \in b ~~\mbox{ if } {|A^T|}_i F^2 = g_i,\mbox{ or
         equivalently } \lambda > 0\\
       v_i \in \overline{b}~~\mbox{ otherwise}.\\
    \end{array}
\right.
\end{equation*}

The \textit{capacity of a $b$-cut} is given by
\begin{equation*}
C = \sum_{ v_i \in b} 2\lambda_i g_i^2 = \sum_{ v_i \in V} 2\lambda_i g_i^2.
\end{equation*}
We can notice that the energy optimized by CCMF \eqref{unsimplified_dual} is not as simple as
the energy minimized by the classic max-flow (graph cuts) algorithm, because the $\lambda$ term is a parameter of the problem that we cannot eliminate.
}

Strang proved that the continuous max-flow problem is dual to the
total variation problem under some assumptions on $g$. Remarkably, this duality is not observed in
the discrete case in which the CCMF dual and CTV problems are given by
\begin{center}
\begin{tabular}{cc}
\\
$\min_{\lambda, \nu} ~\lambda^T g^2 + \frac{1}{4} {\bigg( \1^n
  \cdot/ (|A| \lambda) \bigg)^T \bigg( (c+A\nu)^2 \bigg)}$,~~~~~$\not\equiv$  &  ~~~~~$\min_u  g^T \sqrt{|A^T|{(Au)}^2}, $\\
$\mbox{s. t. } \lambda \geq 0$,~~~~~ & ~~~~~$\mbox{s. t. } \; u_s - u_t = 1.$\\
\\
\end{tabular}
\end{center}

We note that the analytic expressions are different and also not
equivalent.  The duality of the classical max-flow problem with the
minimum cut holds because in the expression of the Lagrangian function,
is possible to deduce a value of $\lambda$ by substituting it into the
dual problem so that it only depends on $\nu$.  However, $\lambda$ in
the dual CCMF problem \eqref{dual_lambda_nu} depends on several values
of neighboring $\lambda$ and $\nu$. Thus, the CCMF dual problem cannot
be simplified by removing a variable, for example by identifying $\nu$
and $u$ in the CTV problem. Said differently, combining a
  null-divergence constraint with an {\em anisotropic} capacity
  constraint in the classical max-flow formulation allows the duality
  to hold; in contrast, in the CCMF formulation, combining a
  null-divergence constraint with an {\em isotropic} capacity
  constraint does not allow such duality to hold.
 
Numerically we can also show on examples that the value $\max F_{st}$
of the flow optimizing CCMF is not equal to the minimum value of CTV
(See Table~\ref{ex}). 
\omitme{In fact, large differences may occur in image
segmentation, as shown in the example of Figure \ref{fig_angio}.}

\begin{table}[hbtp]
\begin{center}
{\small
\begin{tabular}{|p{0.23\textwidth}|p{0.35\textwidth}|p{0.31\textwidth}|}
\hline
\begin{center}  {\bf Transport graph with weights $g$ on the nodes} \end{center} & 
\begin{center} {\bf Combinatorial Continuous Max-Flow (CCMF) } \begin{eqnarray} 
     \nonumber \max_F & F_{st}, \\
     \nonumber \mbox{s. t. } & A^T F = 0,\\
     \nonumber~ & |A^T| F^2 \leq g^2.
\end{eqnarray} \end{center} & 
\begin{center} {\bf Combinatorial Total Variation (CTV)} \end{center} 
\begin{eqnarray}
\nonumber \min_u &~g^T \sqrt{|A^T|{(Au)}^2}, \\
\nonumber \mbox{s. t. } & u_s - u_t = 1.
\end{eqnarray}
\\
\hline
\begin{center}
  \includegraphics[width=0.8\linewidth]{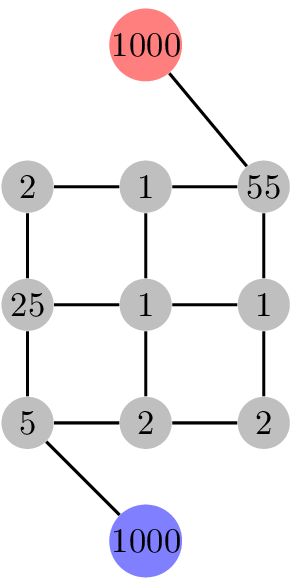} 
\end{center}&
\begin{center}
  \includegraphics[width=1\linewidth]{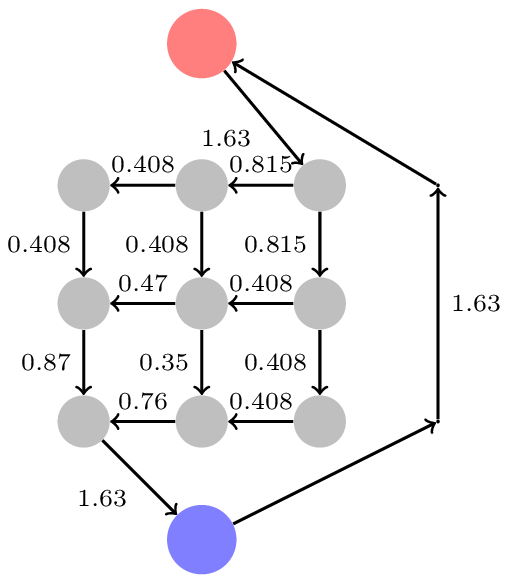} 
\end{center}&
\begin{center}
  \includegraphics[width=0.6\linewidth]{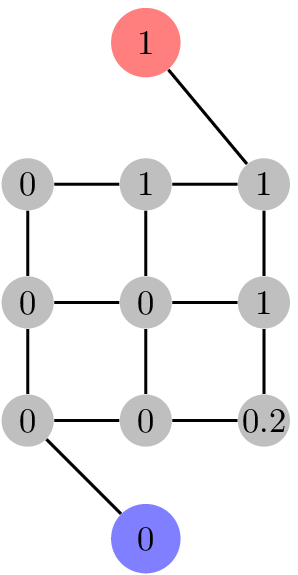} 
\end{center}\\
	\hline 
	~ &
\begin{center} $F_{st} = 1.63$ \end{center} & 
\begin{minipage}{0.3\textwidth}
  {\scriptsize
	\begin{eqnarray*} CTV(u) & = &\sum_i{g_i \sqrt{\sum_j{(u_i-u_j)^2}}}  \\
	 & = & 2 + 1 \times \sqrt{2} + 1 \times \sqrt{2}\\
         & & + 1 \times \sqrt{1+0.8^2}\\
         & & + 2 \times \sqrt{0.8^2+0.2^2}\\
         & & + 2\times 0.2\\
	 & = & 8.16.  
\end{eqnarray*}
}
	\end{minipage}
	\\ 
        \hline
	\end{tabular}
}
	\end{center}
	
	\caption{Example illustrating the difference between optimal
          solutions of combinatorial continuous maximum flow
          problem (CCMF), and combinatorial total
          variation (CTV).}
\label{ex}
\end{table}

 \omitme{
\begin{figure}[ht]
\begin{center}
\begin{tabular}{ccc}
\subfigure[]{\begin{frame}{\includegraphics[width=0.3 \linewidth]{angioseeds.eps}}\end{frame}}~~~
\subfigure[]{\begin{frame}{\includegraphics[width=0.3 \linewidth]{angioTV_43_75.eps}}\end{frame}}~~~
\subfigure[]{\begin{frame}{\includegraphics[width=0.3 \linewidth]{angioCCMF_MF_15_04.eps}}\end{frame}}
\end{tabular}
\end{center}
\caption {An example showing that the segmentations may
    differ significantly depending on whether the CTV or the CCMF
    energy is being optimized --- This example illustrates that CTV
    and CCMF are not dual to each other. (a) Original image with
  foreground and background seeds. (b) Segmentation obtained with
  Split Bregman algorithm of~\cite{Goldstein2009} for CTV optimization,
  (c) CCMF result. }
\label{fig_angio}
\end{figure}
}

\subsubsection{Theoretical links between CCMF and CTV}



Even if CCMF is not dual with CTV, the two problems
are weakly dual. 

In the combinatorial setting the weak duality property is given by 
\begin{equation}
 ||F|| \leq g   \Rightarrow  F^T (A u) \leq g^T ||A u||.
 \end{equation}

The next property shows that the norm $||F||  = |A^T|F^2$ verifies weak duality.

\begin{proposition}
Let $G$ be a transport graph, $F$ a flow in $G$ verifying the capacity
constraint $\sqrt{|A^T| F^2} \leq g$. Let $u$ be a vector of $\R^n$
(defined on nodes of $G$). 
Then
\begin{equation*}F^T Au \leq g^T \sqrt{|A^T| (Au)^2 }.\end{equation*}
\label{prop_weak} 
\end{proposition}

\begin{proof}
Since we know that $\sqrt{|A^T| F^2} \leq g$, the following statement is true
\begin{equation*}\sqrt{|A^T| F^2}^T \sqrt{|A^T| (Au)^2} \leq g^T
  \sqrt{|A^T|(Au^2}.
\end{equation*}
We can now show that
\begin{equation*}F^T Au \leq {\sqrt{(|A^T| F^2)}}^T \sqrt{|A^T|
    (Au)^2}.
 \end{equation*}
\omitme{We recall Cauchy-Schwartz inequality for two vectors $x,y$ of $\R^n$
  \begin{equation*}
\sum_{i=1}^n{x_iy_i} \leq
\sqrt{\left(\sum_{i=1}^n{x_i^2}\right)\left(\sum_{i=1}^n{y_i^2}\right) } .
 \end{equation*}
}
Using summation notation, then by the Cauchy-Schwartz inequality,
\begin{equation*}
\sum_{i\in V}{\sum_{e_{ij}\in E}{F_{ij}(u_i-u_j)}} \leq \sum_{i\in
  V}{\sqrt{\left( \sum_{e_{ij}\in E}{{F_{ij}}^2}\right) \left(\sum_{e_{ij}\in E}{(u_i-u_j)^2}\right)}}.
 \end{equation*}
We conclude that
\begin{equation*} F^T Au \leq {\sqrt{(|A^T| F^2)}}^T
  \sqrt{|A^T| (Au)^2} \leq g^T \sqrt{|A^T| (Au)^2 }.\end{equation*}
\end{proof}

In terms of energy value, this proposition means that the CCMF energy,
that is to say the flow, is always smaller than the combinatorial
total variation.

\begin{proposition}
Let $F$ be a compatible flow verifying the constraints in \eqref{f_new}, and $u$ a vector of $\R^n$
such as $u_s-u_t=1$.
Then, \begin{equation}  F_{st} \leq g^T\sqrt{A^T(Au)^2}.\end{equation}
\end{proposition}

\begin{proof}
 In the combinatorial setting, the Green formula gives
\begin{equation*} 
F^T A u = u^T A^T F.
\end{equation*}
Let $a$ be a vector of $\R^n$ defined for each node $v_i$ as 
\begin{equation} 
\begin{gathered}
a_i = \left\{
    \begin{array}{lll}
        -1 &\mbox{ if } v_i \mbox{ belongs to the sink,} \\
        ~1 &\mbox{ if } v_i \mbox{ belongs to the source,} \\
        ~0 &\mbox{  otherwise.}\\
    \end{array}
\right.
\label{eq:b}
\end{gathered}
\end{equation}
We can provide the following equivalent formulation of the CCMF problem~\eqref{f_new} 
  using an incidence matrix $A$ of the transport graph deprived of the
  sink-source edge:
\begin{align}
\begin{gathered}
\max {F_{st}},\;\;\;\;\;\;\\
\text{s.t.}\;\;\;  A^T F = F_{st} a,\\
 |A^T| F^2 \leq g.
 \end{gathered}
 \end{align}
 As $F$ verifies the divergence free constraint $A^TF=F_{st} a$,
 \begin{equation*} 
u^T A^T F = u^T F_{st} a,
\end{equation*}
and as the equation $u_s-u_t=1$ may be written equivalently $a^Tu=1$, we conclude that
 \begin{equation*} 
 u^T F_{st} a = F_{st}, 
\end{equation*}
and by weak duality (Property~\ref{prop_weak}),  
\begin{equation*} 
F_{st} \leq g^T\sqrt{A^T(Au)^2}.
\end{equation*}
\end{proof}

This property should be of interest for extending CCMF to applications other
than segmentation, which is outside the scope of this paper. In the
next section, we present the performance of our approach in the context of segmentation.

\section{Results} 

We now present applications of Combinatorial Continuous Maximum Flow
in image segmentation. To be used as a segmentation algorithm, three
solutions are possible: the dual variable $\nu$ may be used directly
if the user needs a matted result, otherwise $\nu$ may be thresholded,
or finally an isoline or isosurface may be extracted from $\nu$.

\begin{figure}[h]
\begin{center}
\begin{tabular}{cccc}
\subfigure[]{\includegraphics[width=0.24 \linewidth]{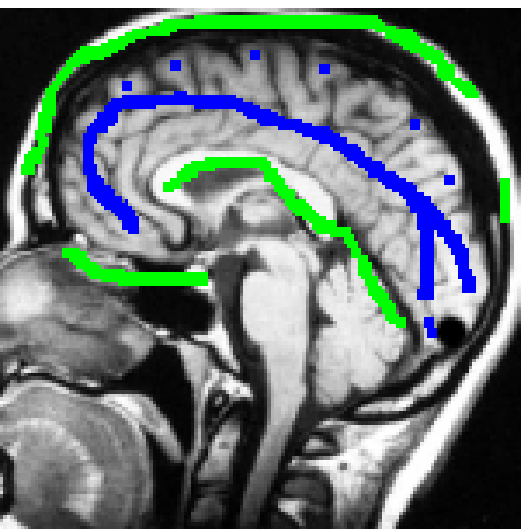}}
\subfigure[]{\includegraphics[width=0.24
    \linewidth]{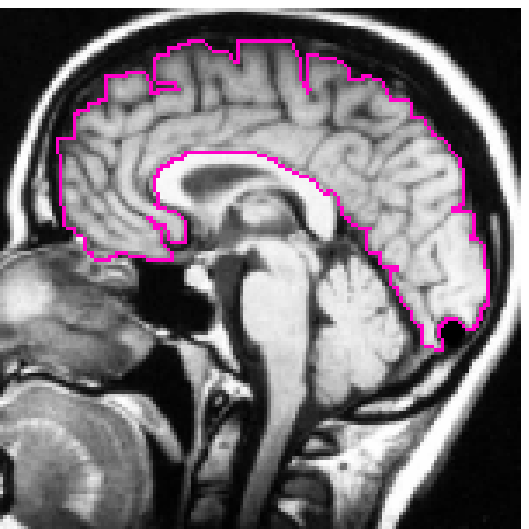}}
\subfigure[]{\includegraphics[width=0.24
    \linewidth]{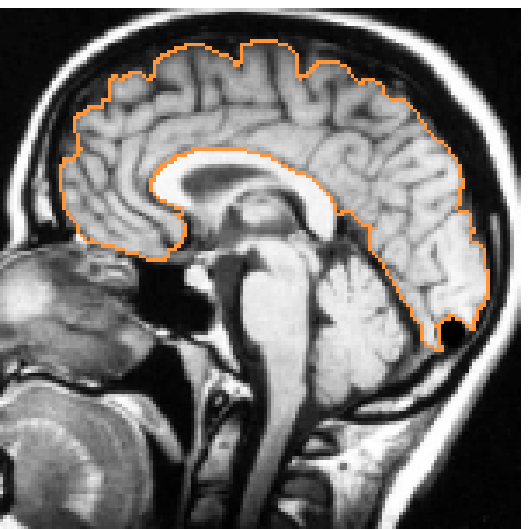}}
\subfigure[]{\includegraphics[width=0.24
    \linewidth]{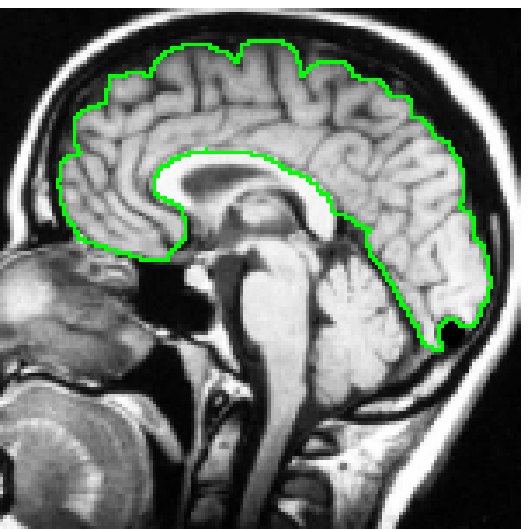}}
\end{tabular}
\begin{tabular}{cccccc}
{\scriptsize
\begin{tabular}{cc}
\includegraphics[width=0.1 \linewidth]{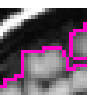}\\GC
\end{tabular}\begin{tabular}{c}\includegraphics[width=0.1 \linewidth]{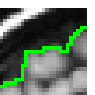}\\CCMF\end{tabular}~~~~
\begin{tabular}{c}\includegraphics[width=0.12 \linewidth]{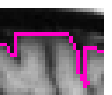}\\GC\end{tabular}
\begin{tabular}{c}\includegraphics[width=0.12 \linewidth]{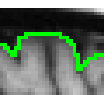}\\CCMF\end{tabular}~~~~
\begin{tabular}{c}\includegraphics[width=0.1 \linewidth]{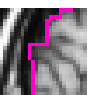}\\GC\end{tabular}
\begin{tabular}{c}\includegraphics[width=0.1 \linewidth]{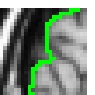}\\CCMF\end{tabular}}
\end{tabular}
\end{center}
\caption {Brain segmentation. (a) Original image with foreground and
  background seeds. (b,c,d) Segmentation obtained with (b) graph cuts (GC),
  (c) AT-CMF, threshold of $P$ (obtained after 10000 iterations), (d)
  CCMF, threshold of $\nu$ (15 iterations).
}
\label{fig_brain}
\end{figure}

In the introduction we discussed how several works are related to
CMF. Some are equivalent or slight modifications of AT-CMF for
non-segmentation applications~\cite{unger2008tvseg,
  ungerReport,pock2008ECCV,zach2008}. As in the original AT-CMF work,
none of these come with a convergence proof. In contrast, the work of
\cite{Pock2009} and~\cite{Lippert2006} are provably convergent but
both are very slow, and do not generalize easily to 3D image
segmentation. Consequently, it seems reasonable to compare our
segmentation method to the original AT-CMF as representative of the
continuous approach, as well as graph cuts, which represent the purely
discrete case. Finally, even if the total variation problem is
  different than the CMF problem (existence of a duality gap in the
  continuous \cite{nozawa1994duality}), the two problems are related
  enough to be compared. Specifically, we will also compare with Chambolle
  and Pock's recent work~\cite{ChambollePock2010} which presented an
  algorithm for optimizing a TV-based energy for image segmentation.

Our validation is intended to establish three properties of the CCMF algorithm.
First, we establish that the CCMF does avoid the metrication artifacts exhibited
by conventional graph cuts (on a 4-connected lattice).  This property is
established by examples on a natural image and the recovery of the classical
catenoid structure as the minimal surface spanning two rings. Second, we compare
the convergence of the CCMF algorithm to the AT-CMF algorithm to show that the
CCMF algorithm converges more quickly and in a more stable fashion.
Finally, we establish that our formulation of the CMF problem does not degrade
segmentation performance on a standard database. In fact, because of the
reduction in metrication error our algorithm gives improved numerical
results.
For this experiment, we use the GrabCut database to compare the quality of the
segmentation algorithms. In addition to the above tests, we demonstrate through
examples that the CCMF algorithm is also flexible enough to incorporate prior
(unary) terms and to operate in 3D. 

\begin{figure}[ht]
\begin{center}
\begin{tabular}{ccc}
 \subfigure[]
    {\begin{frame}{\includegraphics[width=0.3 \linewidth]{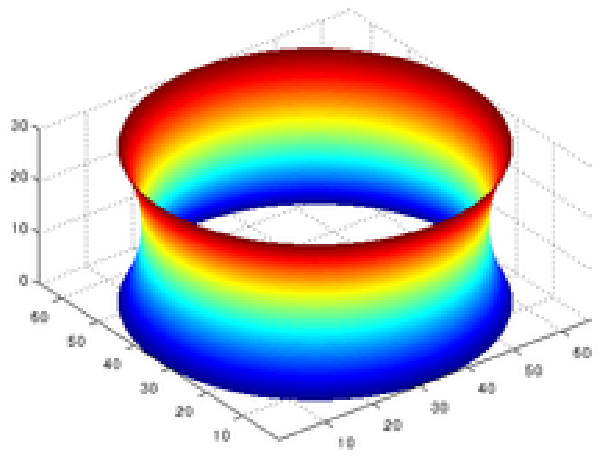}}\end{frame}}
 \subfigure[]
    {\begin{frame}{\includegraphics[width=0.3 \linewidth]{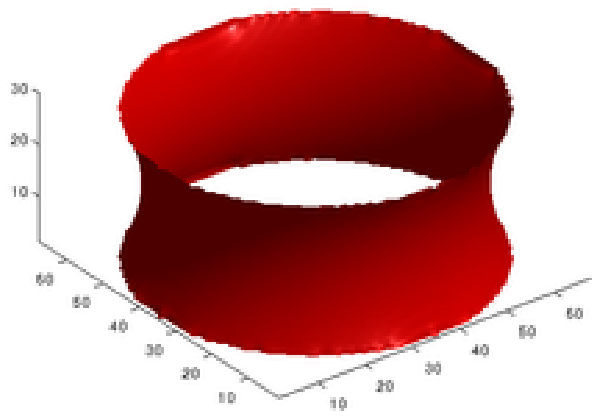}}\end{frame} }
 \subfigure[]
    {\begin{frame}{\includegraphics[width=0.3
      \linewidth]{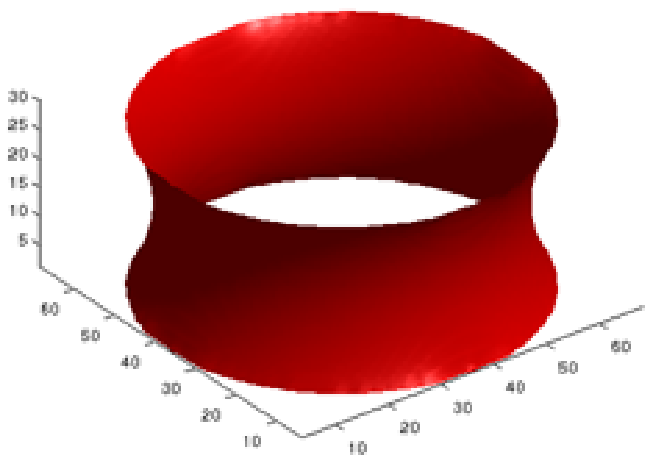}}\end{frame}  }
\end{tabular}
\end{center}
\caption {The catenoid test problem: The source is constituted by two
  full circles and sink by the remaining boundary of the image. (a)
  Surface computed analytically, (b) isosurface of $P$ obtained by
  CMF, (c) isosurface of $\nu$ obtained by CCMF. The root mean square
  error (RSME) has been computed to evaluate the precision of the
  results to the surface computed analytically. The RSME for CMF is
  1.98 and for CCMF 0.75. The difference between those results is due
  to the fact that the CMF algorithm enforces exactly the source and
  sink points, leading to discretization around the disks. In
  contrast, the boundary localized around the seeds of $\nu$ is
  smooth, composed of grey levels. Thus the resulting isosurface
  computed by CCMF is more precise.}
\label{fig_cate}  
\end{figure}

\subsection{Metrication artifacts and minimal surfaces}

We begin by comparing the CCMF segmentation result with the classical
max-flow algorithm (graph cuts).  Figure \ref{fig_brain} shows the
segmentation of a brain, in which the contours obtained by graph cuts
are noticeably blocky in the areas of weak gradient, while the
contours obtained by both AT-CMF and CCMF are smooth.

In the continuous setting, the maximum flow computed in a 3D volume
produces a minimal surface. The CCMF formulation may be also
recognized as a minimal surface problem. In the dual formulation, the
objective function is equivalent to a weighted sum of surface nodes.
In \cite{hugues}, Appleton and Talbot compared the surfaces obtained
from their algorithm with the analytic solution of the catenoid
problem to demonstrate that their algorithm was a good approximation
of the continuous minimal surface and was not creating discretization
artifacts.  The catenoid problem arises from consideration of two
circles with equal radius whose centers lie along the $z$ axis. The
minimal surface which forms to connect the two circles is known as a
catenoid. The catenoid appears in nature, for example by creating a
soap bubble between two rings. In order to demonstrate that CCMF is
also finding a minimal surface, we performed the same catenoid
experiment as was in \cite{hugues}.  The results are displayed in
Figure \ref{fig_cate}, where we show that CCMF approximates the
analytical solution of the catenoid with even greater fidelity than
the AT-CMF example.

\begin{figure}[ht]
\begin{center}
\begin{tabular}{ccc}
\subfigure[Input seeds]{
\begin{tabular}{c}
\includegraphics[width=0.17
  \linewidth]{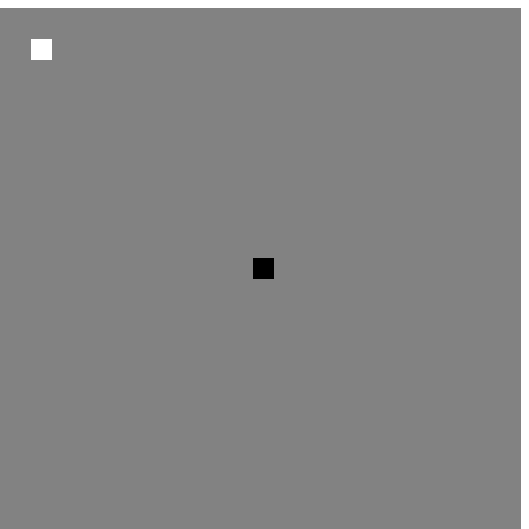}
\end{tabular}}&
\subfigure[Input images]{
\begin{tabular}{ccc}
 \includegraphics[width=0.056 \linewidth]{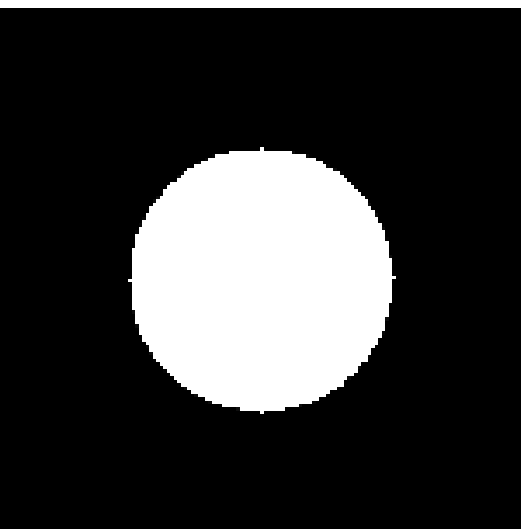}&
 \includegraphics[width=0.056 \linewidth]{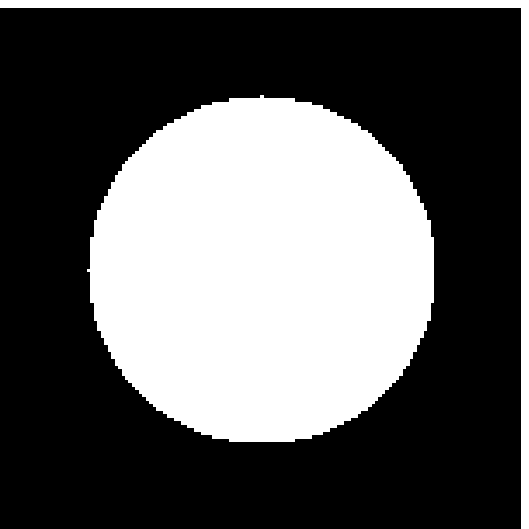}&
 \includegraphics[width=0.056 \linewidth]{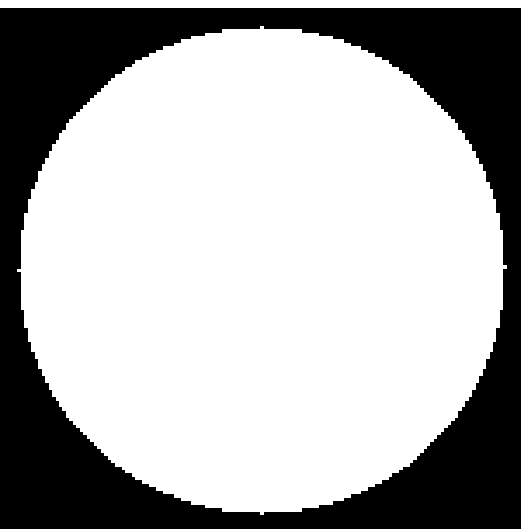}\\
\includegraphics[width=0.056 \linewidth]{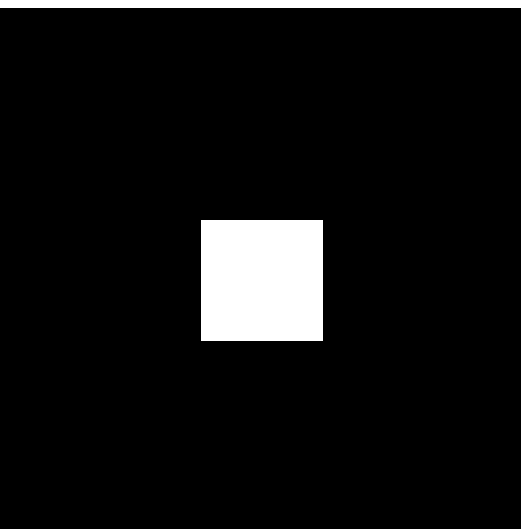}&
 \includegraphics[width=0.056 \linewidth]{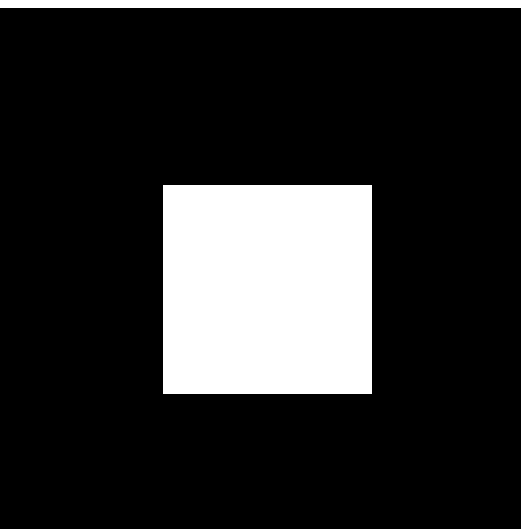}&
 \includegraphics[width=0.056 \linewidth]{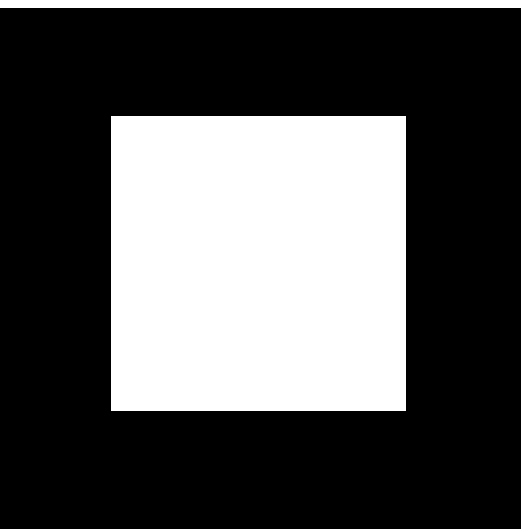}\\
\includegraphics[width=0.056 \linewidth]{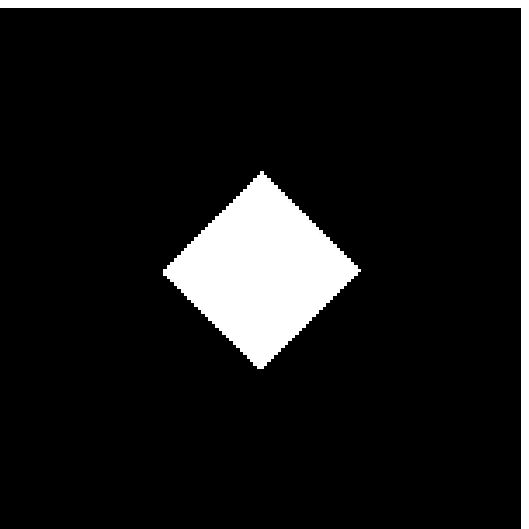}&
 \includegraphics[width=0.056 \linewidth]{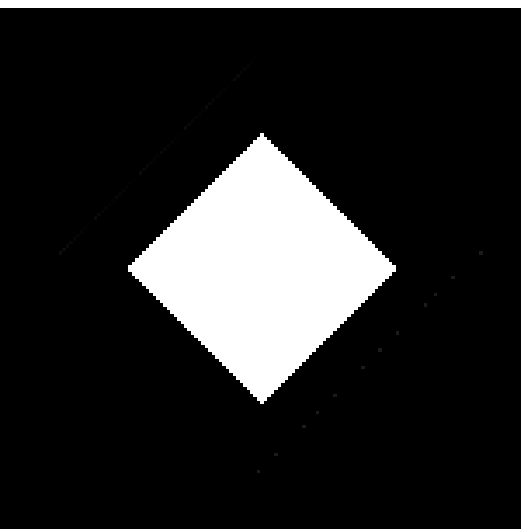}&
 \includegraphics[width=0.056 \linewidth]{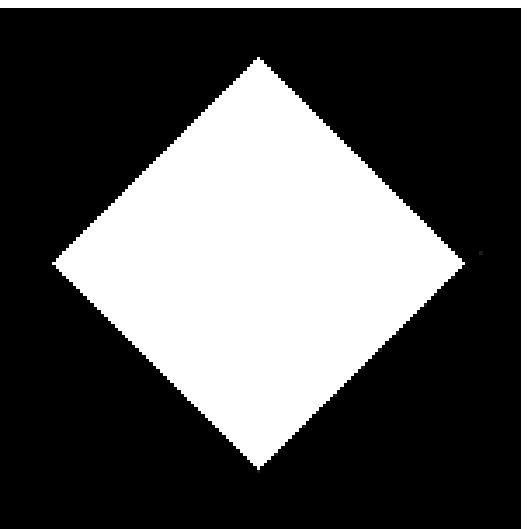}
\end{tabular}}&
\begin{tabular}{c}
\subfigure[CCMF energy vs true perimeters]{\includegraphics[width=0.5 \linewidth]{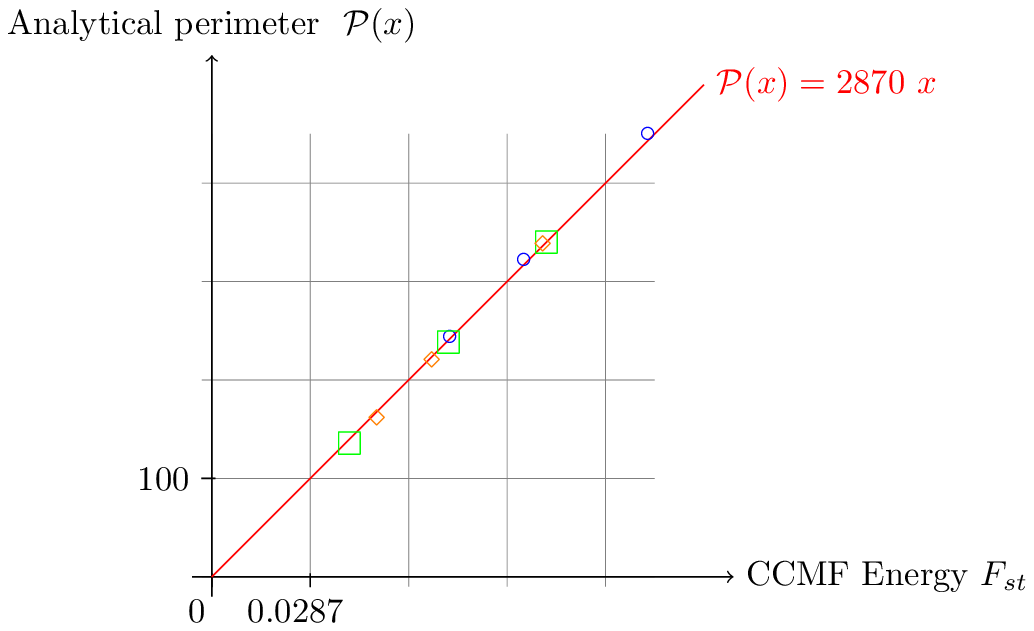}}
\end{tabular}
\end{tabular}

\begin{tabular}{cc}
\subfigure[Graph cut cost for lines of same length]{
\begin{tabular}{cc}
\includegraphics[width=0.11
    \linewidth]{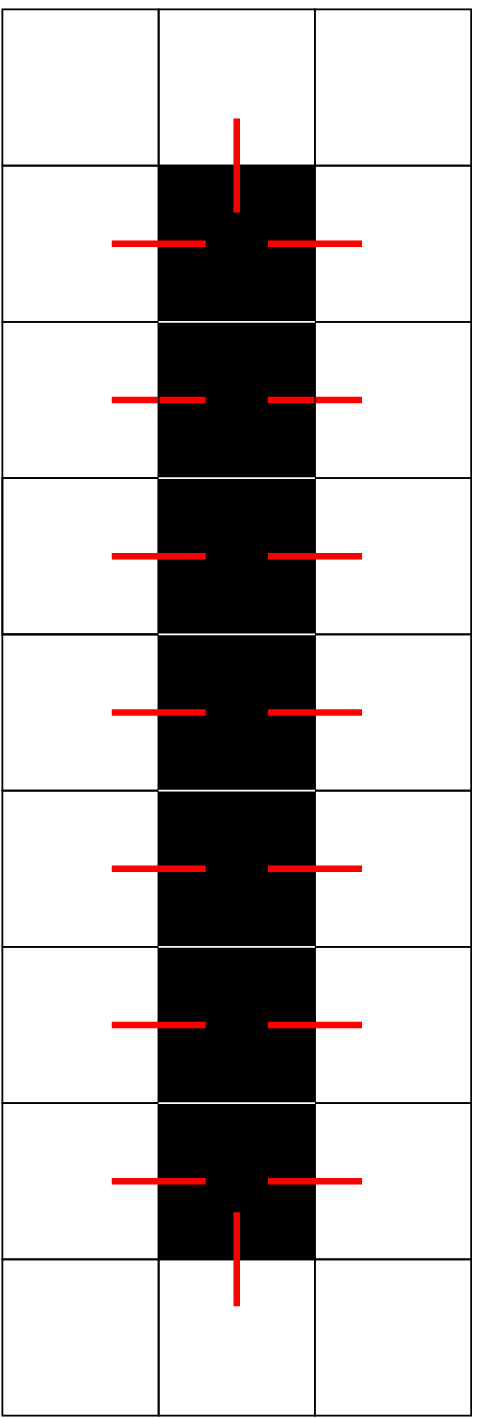}& \includegraphics[width=0.24
    \linewidth]{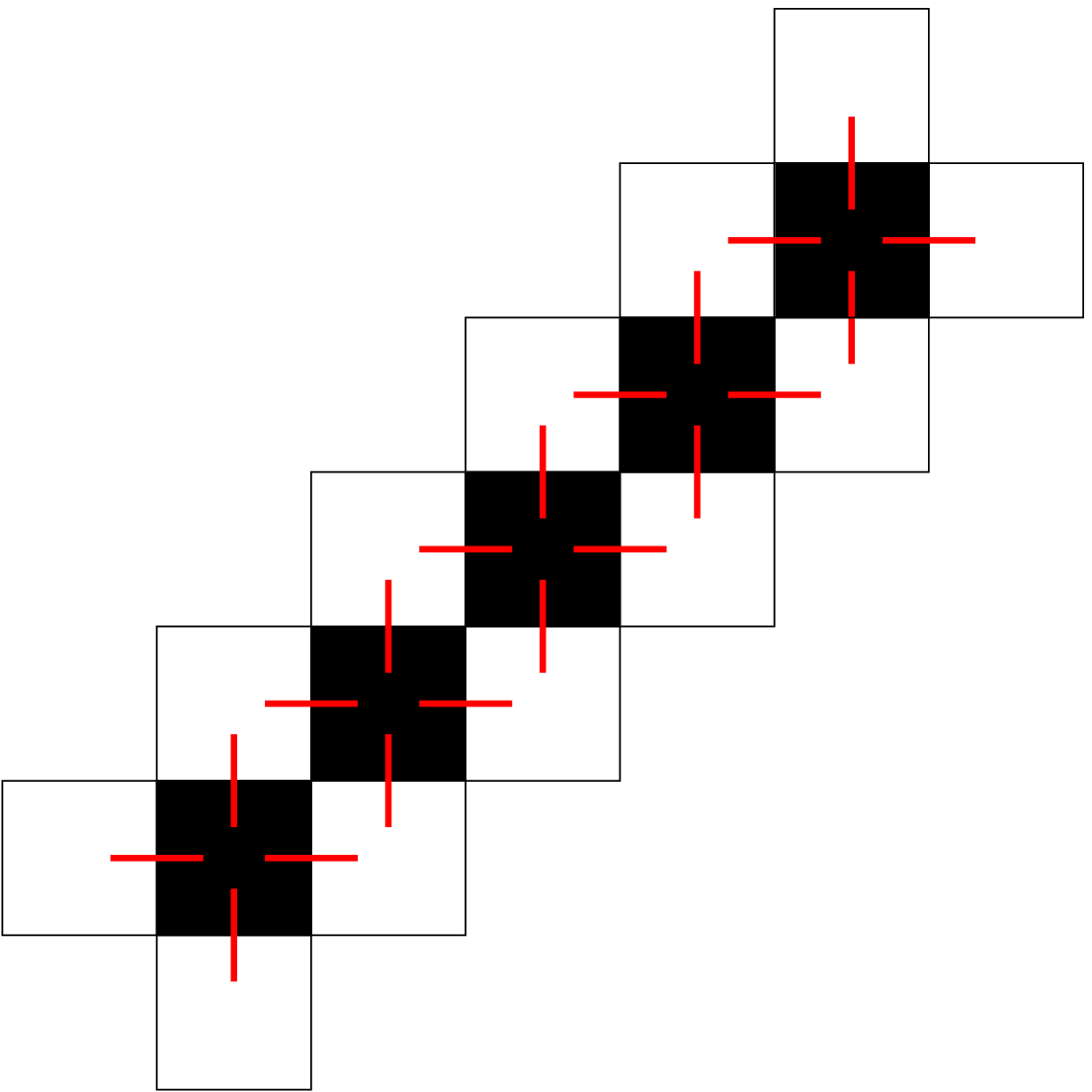}\\
GC cut cost : 16 & GC cut cost : 20 
\end{tabular}
 } &
\begin{tabular}{c}
\subfigure[GC energy vs true perimeters]{\includegraphics[width=0.5 \linewidth]{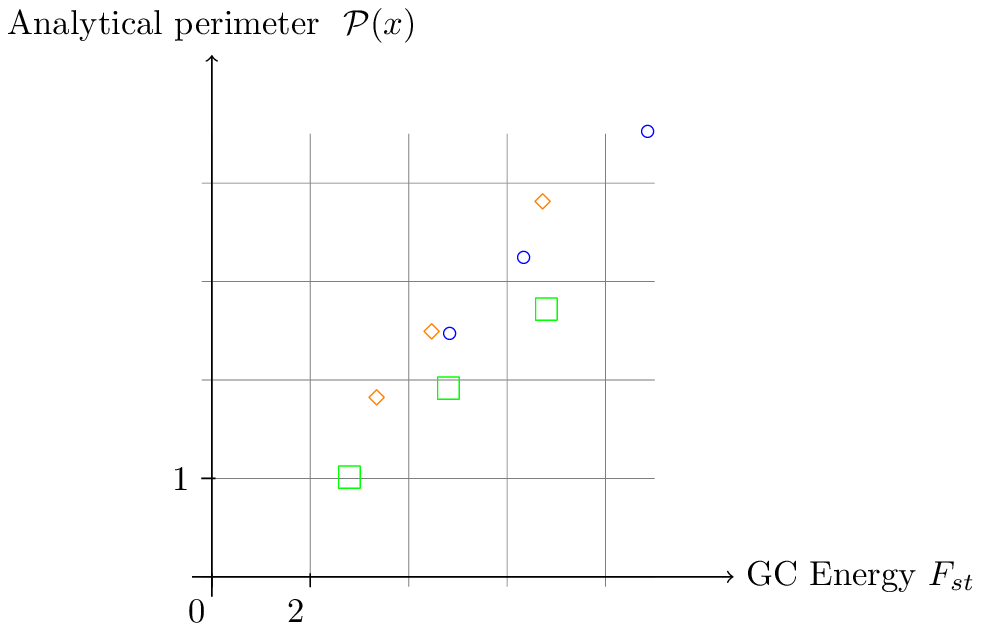}}
\end{tabular}
\end{tabular}
\end{center}
\caption {\change{Comparison of perimeters computed analytically
    and with Graph cuts and CCMF. To estimate the boundary length with
    CCMF, segmentations have been produced using the input images in
    (b) with the foreground and background seeds that appears in
    (a). (c) : plot showing the linear relation between the true
    perimeters and the energies obtained with CCMF. The symbols used
    to represent the dots correspond to the objects segmented in the
    images.  A similar relation of
    proportionality is obtained with CP-TV. (d) : value of the cut obtained for a diagonal and a
    vertical line of same length using graph cuts. The difference
    observed explains (e) : plot showing the nonlinear relation
    between the true perimeters and the energies obtained with Graph
    cuts.}}
\label{test_perimeter} 
\end{figure}

In order to verify that the solution obtained by CCMF
  approximates perimeters of planar objects in 2D, we segmented
  several shapes -- squares, discs -- with known perimeters, scaled at
  different size, and compared the analytic perimeters with the
  energies obtained by CCMF. The results are reported in
  Fig.~\ref{test_perimeter}, where we observe that the true boundary
  length is proportional to the energy obtained with CCMF.

\subsection{Stability, convergence and speed}

We may compare the segmentation results using $\nu$ to
Appleton-Talbot's result using $P$.  We recall that AT-CMF solves the
partial differential equation system \eqref{AT} in order to solve the
continuous maximum flow problem \eqref{eq:cont_MF}, but no proof was
given for convergence. The potential function $P$ approximates an
indicator function, with $0$ values for the background labels, and $1$
for the foreground. It can be difficult to know when to stop the
AT-CMF algorithm, since the iterations used to solve the AT-CMF
algorithm may oscillate, as displayed in Figure \ref{fig_diag} on a
synthetic image.  In contrast, the CCMF algorithm is guaranteed to
converge and smoothly approaches the optimum solution.

We also compare segmentation results with results obtained by
  the recent algorithm of Chambolle and Pock~\cite{ChambollePock2010}
  optimizing a total variation based energy using a dual variable. In
  comparison to our work, the dual variable is not a flow (no
  constraint on the divergence and no proof of convergence toward a
  minimal surface as defined in \cite{strang1983} or
  \cite{hugues}). In the literature, the continuous-domain duality
  between TV and max-flow is assumed, but in computational practice
  they are really different. Nevertheless, even as the method of
  Chambolle-Pock is solving a different problem, we performed
  numerical comparative tests. The problem of Chambolle and Pock (now
  denoted CP-TV) for binary segmentation optimizes
\begin{equation}
\min_{u\in[0,1]^n} \max_{F\in\R^m, ||F||_{\infty}\leq 1} {F^T(Au) +
  g^Tu} + \mbox{ hard constraints attachment},
\end{equation}
where $A \in \R^{m\times n}$ is the incidence matrix of the graph of
$n$ nodes and $m$ edges defined on the image, $g\in\R^n$ is the metric
on the nodes, defined in equation \eqref{eq_exp_metric}.

The CCMF algorithm is faster than both AT-CMF and CP-TV for 2D image
segmentation. We have implemented CCMF in Matlab. We used an
implementation of Appleton-Talbot's algorithm in {\tt C++} provided by
the authors, and a Matlab software implementing CP-TV on CPU, also
provided by the respective authors. The three implementations make use
of multi-threaded parallelization, and the CPU times reported here
were computed on a Intel Core 2 Duo (CPU 3.00GHz) processor, with 2 Gb
of RAM. The CCMF average computation time on a $321 \times 481$ image
is 181 seconds after 21 iterations. For AT-CMF, 80000 iterations
require 547 seconds. For CP-TV, {\change the average number of
  iterations needed for convergence is 36000, requiring an average
  time of 1961 seconds on CPU, and the median computation time is 1406
  seconds.} Note that a GPU implementation of CP-TV, also provided by
the authors, achieves a speed gain factor of about 100 on a NVidia
Quadro 3700. Although it was not pursued here, it should be noted that
the CCMF optimization approach could also fully benefit from
parallelization (e.g., on a GPU) because the core computation is to
solve a linear system of equations. Realizing the benefit of a GPU
solution would require the use of an iterative algorithm such as
conjugate gradients or multigrid~\cite{bolz2003sparse,kruger2003GPU},
which would make comparisons to the direct solver used here less
immediate.  However, previous examples in the literature have
suggested that this change in solver has not created difficulties in
order to realize the benefits of a GPU
architecture~\cite{grady2005GPU}.

Sometimes, it may not always be necessary to wait so long for the
complete convergence of CCMF (or AT-CMF) to obtain acceptable
results. In cases where images exhibiting sufficiently strong
gradients, one iteration may be enough to obtain a satisfying
segmentation. On such an image, one iteration of CCMF, and 100
iterations of AT-CMF show acceptable approximate results reached only
after about 2 seconds for either algorithm. However, one cannot always
rely on strong edges, and the power of our CCMF formulation is to
behave in a predictable fashion even when edges are weak.

We may also compare the computation time of CCMF to the optimization of total
variation using the Split Bregman method~\cite{GoldsteinOsher}.  Optimizing TV
on a $100 \times 100$ image requires 5000 iterations and takes 23 seconds with a
sequential implementation of Split Bregman in C. On the same image, CCMF
required only 17 iterations to reach the convergence criteria $||r_d|| < 1$ and
$||\widehat{\eta}|| < 2$, taking 4.7 seconds with a Matlab implementation (using
a 2-threaded solver). We conclude that TV-based methods appear to be
quite slow in the context of image segmentation, although we employed the Split
Bregman algorithm which is known as one of the fastest algorithms for TV
optimization for image denoising, and the very recent CP-TV
  algorithm. This difference in speed between the denoising and segmentation
applications is due to the very large number of iterations required to
convergence to a binary segmentation.

\begin{figure}[ht]
\begin{center}
\begin{tabular}{ccccccc}
 \subfigure[]{\begin{frame}{\includegraphics[width=0.23
         \linewidth]{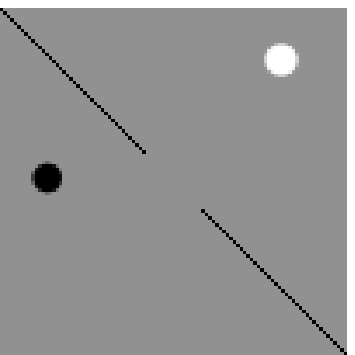}}\end{frame}}
\subfigure[]{\begin{frame}{\includegraphics[width=0.23 \linewidth]{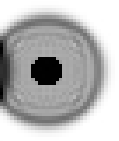}}\end{frame} }
\subfigure[]{\begin{frame}{\includegraphics[width=0.23 \linewidth]{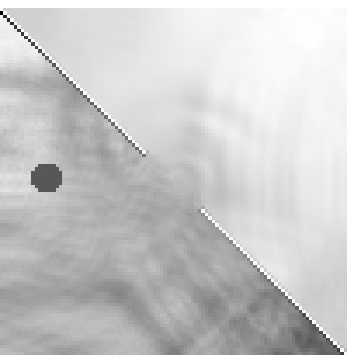}}\end{frame} }
\subfigure[]{\begin{frame}{\includegraphics[width=0.23
         \linewidth]{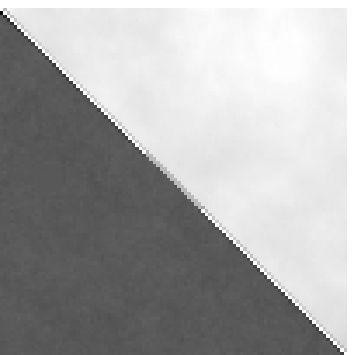}}\end{frame} }\\
 \subfigure[]{\begin{frame}{\includegraphics[width=0.23
         \linewidth]{diag100_withseeds.eps}}\end{frame}}
\subfigure[]{\begin{frame}{\includegraphics[width=0.23 \linewidth]{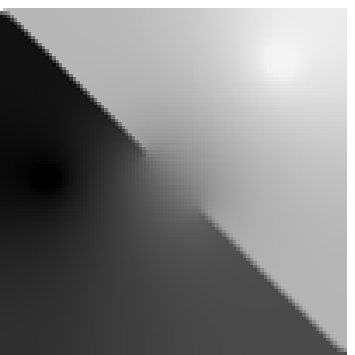}}\end{frame} }
\subfigure[]{\begin{frame}{\includegraphics[width=0.23 \linewidth]{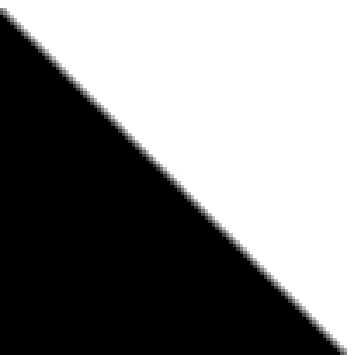}}\end{frame} }
\subfigure[]{\begin{frame}{\includegraphics[width=0.23 \linewidth]{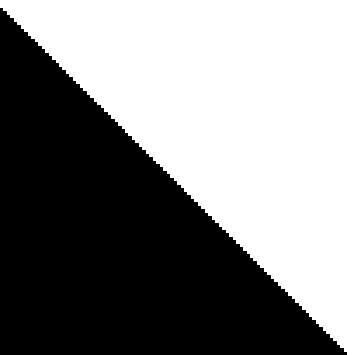}}\end{frame} }
\end{tabular}
\end{center}
\caption {Segmentation of an artificial image with AT-CMF (top row) and
  CCMF (bottom row). Top row (AT-CMF): (a) Image where the black and white discs are seeds. AT-CMF result stopped in (b) after 100 iterations, (c) 1000
  iterations, (d) 10000 iterations. Bottom row (CCMF): (e) Image where the black and white discs are
  seeds, CCMF result $\nu$ after (f) 1 iterations, (g) $\nu$ after 15 iterations
  iterations and (h) threshold of the final $\nu$.}
\label{fig_diag} 
\end{figure}

\subsection{Segmentation quality}

In this experiment, we compared our CCMF formulation to the AT-CMF
formulation and conventional graph cuts on the problem of image
segmentation, to determine if there were strong performance differences between
the formulations.  We expect that there would not be, since in
principle all three formulations are trying to ``minimize the cut'',
but have slightly different definitions of the cut length.  The
primary advantage that we expect from our formulation is in the
reduction of metrication artifacts (as compared to conventional graph
cuts) and speed/convergence (as compared to AT-CMF).  

Our experiment consists of testing the Graph Cut, AT-CMF, CP-TV, and
Combinatorial Continuous Max-Flow algorithms on a database with the same
seeds. We used the Microsoft `GrabCut' database available online
\cite{rother2004grabcut}, which is composed of fifty images provided with
seeds. From the seed images provided, it is possible to extract two different
possible markers for the background seed. Examples of such seeds are shown in
Fig.~\ref{DB_results_a}. Different convergence criteria are available for CCMF,
such as the duality gap and norms of the residuals. However, for the CMF
algorithm, we do not have any satisfying criteria. Bounding the number of non
binary occurrences of $P$ does not mean that the convergence is reached, because
of possible oscillations. An intermediate result after ten thousand iterations
may be significantly different from the result reached after one hundred
thousand. Consequently, we have run the AT-CMF algorithm until we were convinced
to have reached convergence, {\em i.e.} when $P$ was nearly binary and did not
change significantly when we doubled the number of iterations.  For half of the
images in the GrabCut database, AT-CMF algorithm required more than 20000
iterations to reach convergence, for a third of the images, more than 80000
iterations, and for 1/4 of the images, more than 160000 iterations. Binary
convergence was still not reached after even 500000 iterations for the rest of
the images (1/4), but we stopped the computation anyway. The
  Chambolle-Pock TV based algorithm for binary segmentation is provably
  convergent and benefits also from several possible convergence criteria. The
  convergence criteria used is a ratio of changed pixels from two successive
  intervals of iterations being smaller than an epsilon that we fixed to
  $10^{-7}$ in order to obtain satisfying results. In practice, 1/4 of the
  images converged in less than 60000 iterations and for the rest of them we
  increased the maximum number of iterations up to 200000. In contrast, CCMF
only needed 21 iterations on average, and never more than 27, to reach the
convergence criteria $||r_d|| < 1$ and $||\widehat{\eta}|| < 2$.  We noted that
for all methods, segmentation quality degraded quickly if these conditions were not respected.

\begin{figure}[bt]
\label{DB_results} 
\begin{center}
\begin{tabular}{ccc}
\includegraphics[width=0.32 \linewidth]{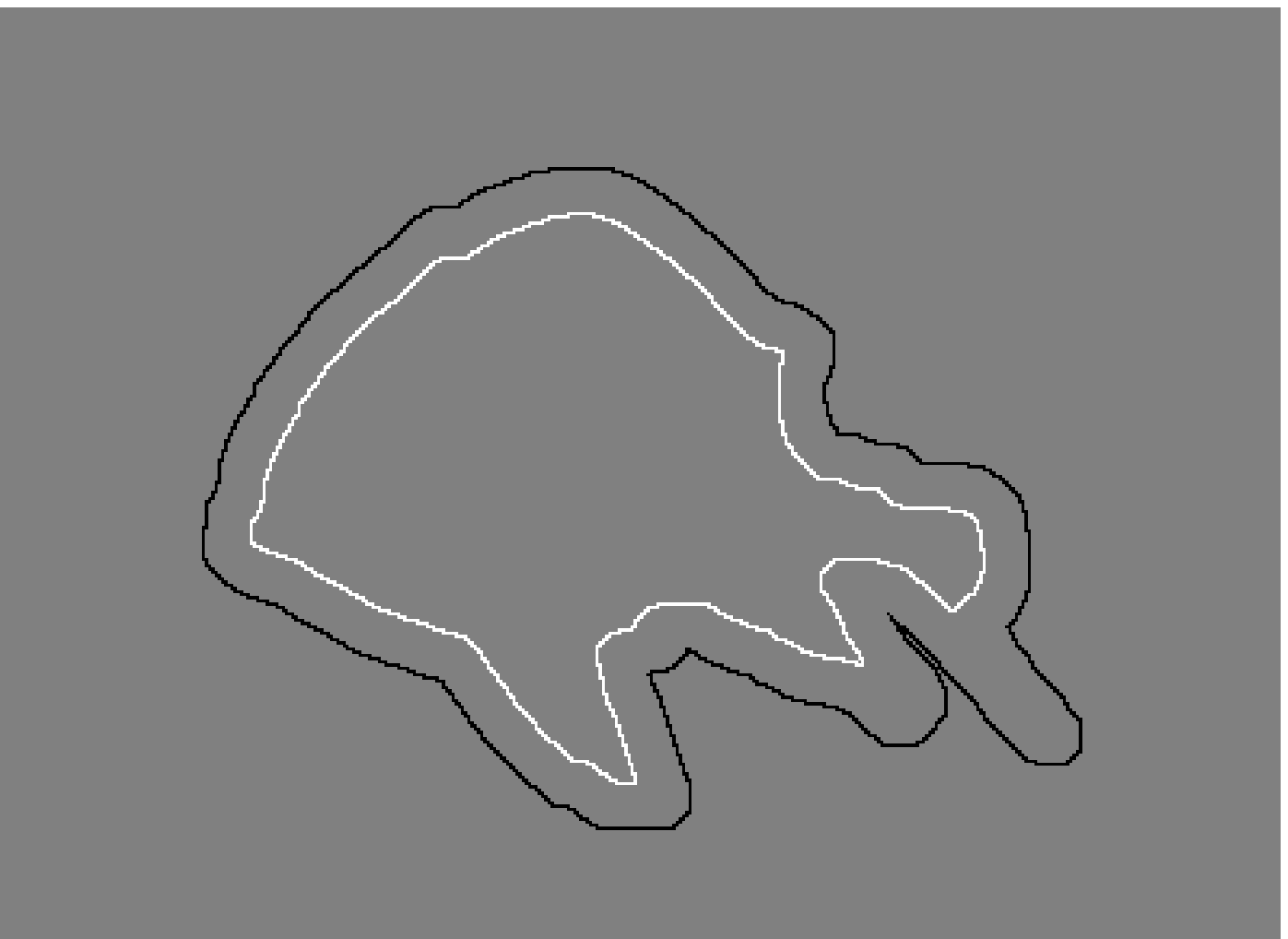}
\includegraphics[width=0.3 \linewidth]{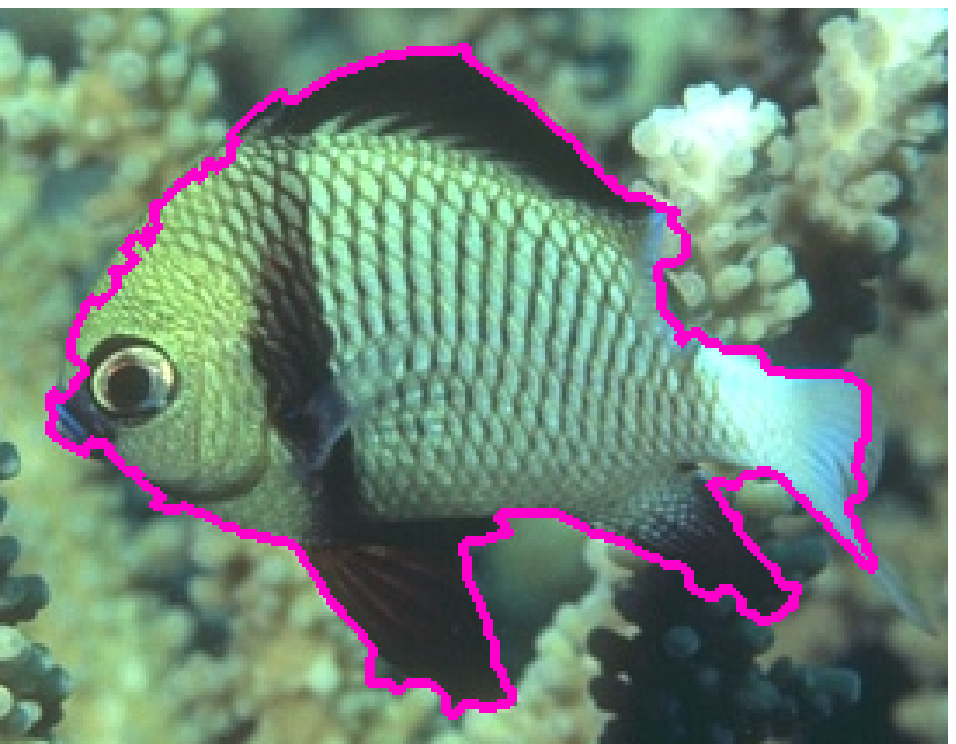}
\includegraphics[width=0.3 \linewidth]{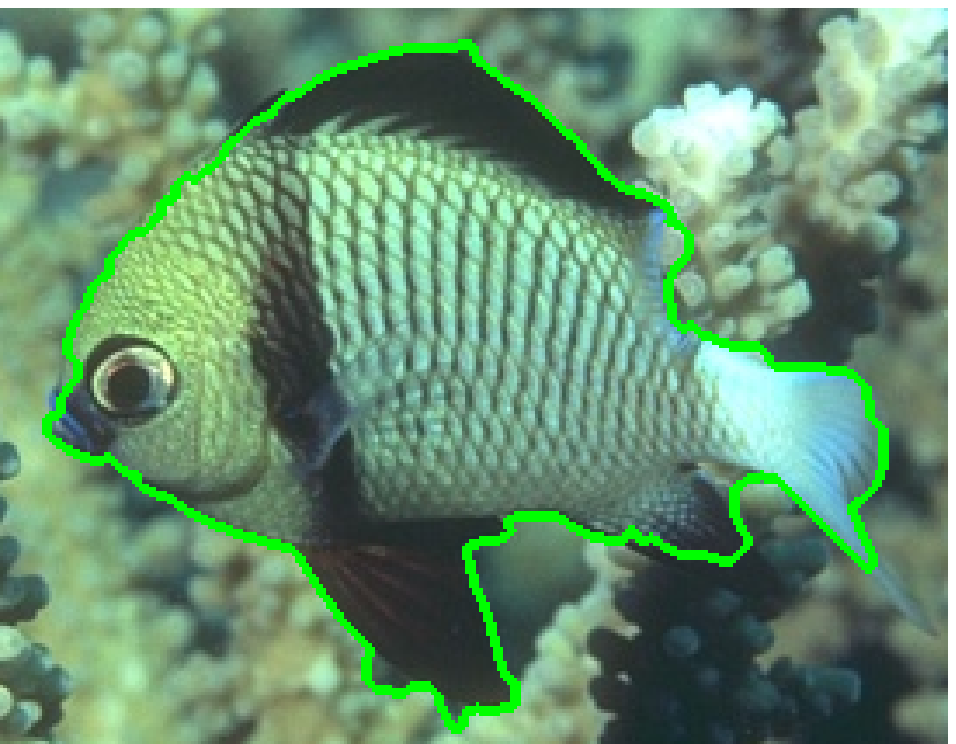}\\
\subfigure[\label{DB_results_a}]{\includegraphics[width=0.32 \linewidth]{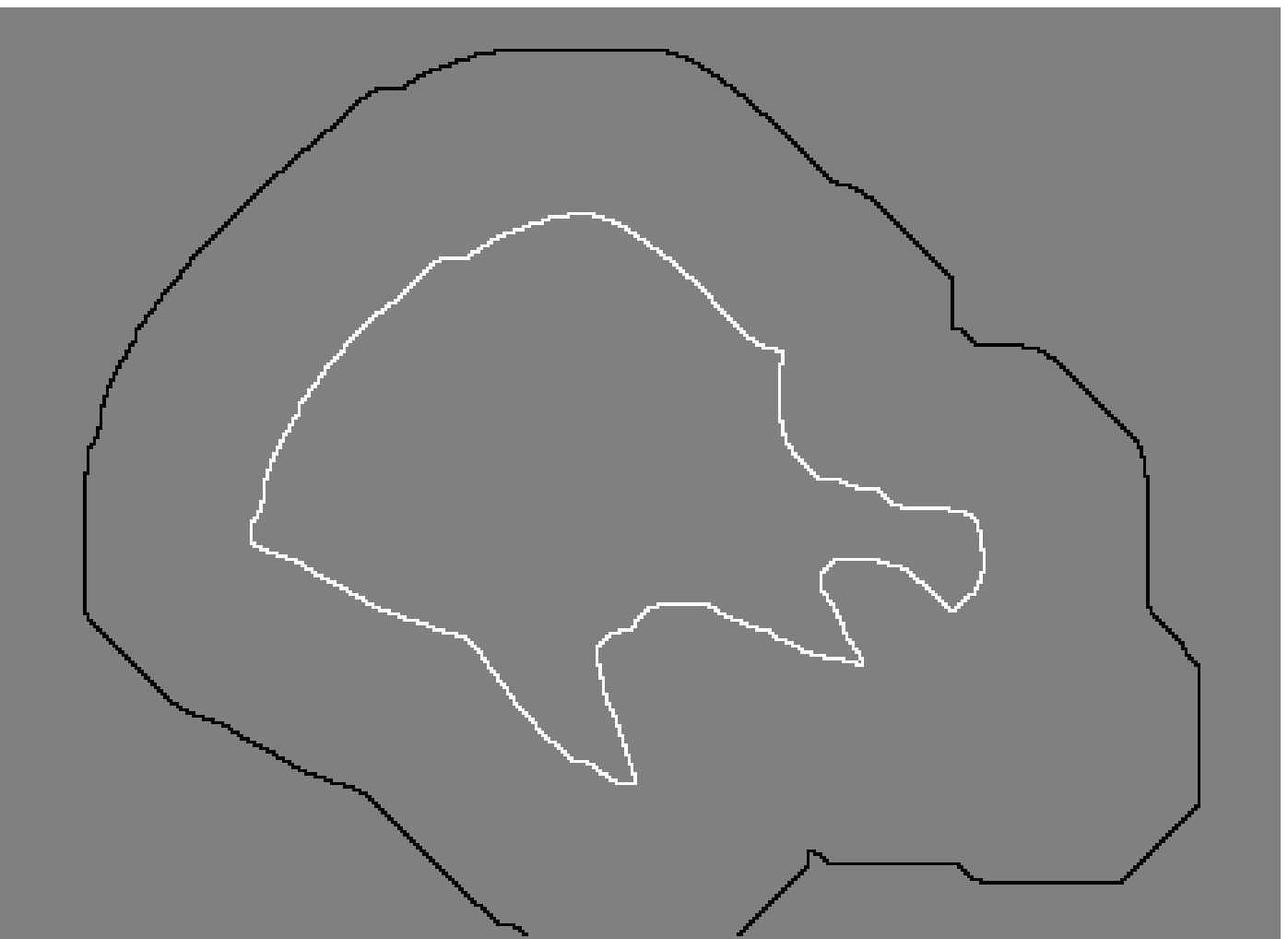}}
\subfigure[]{\includegraphics[width=0.3 \linewidth]{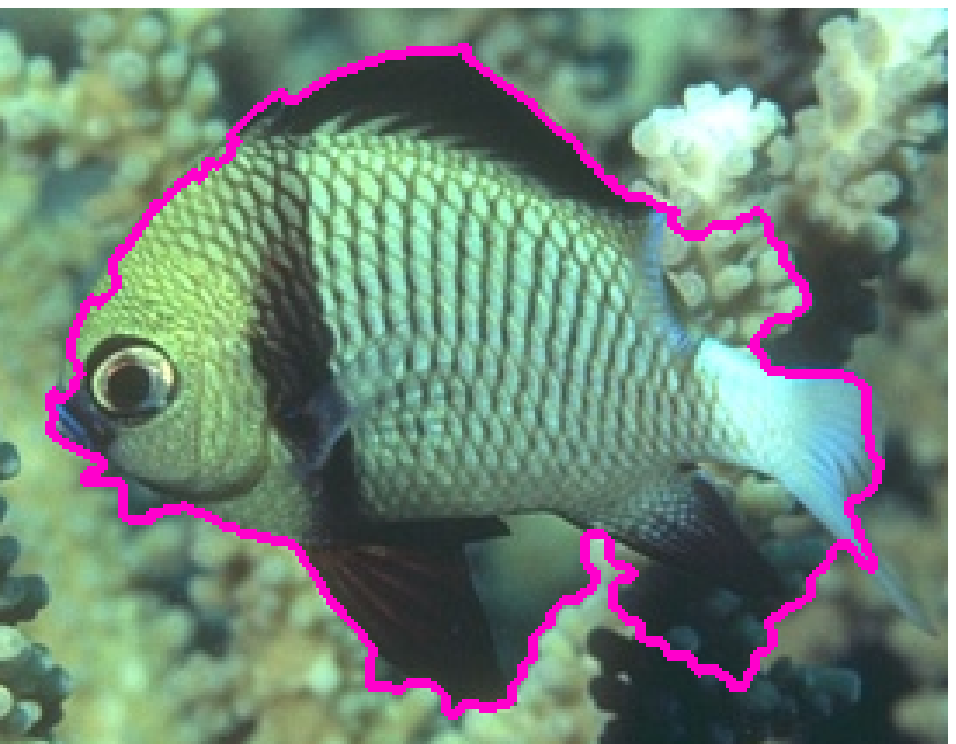}}
\subfigure[]{\includegraphics[width=0.3 \linewidth]{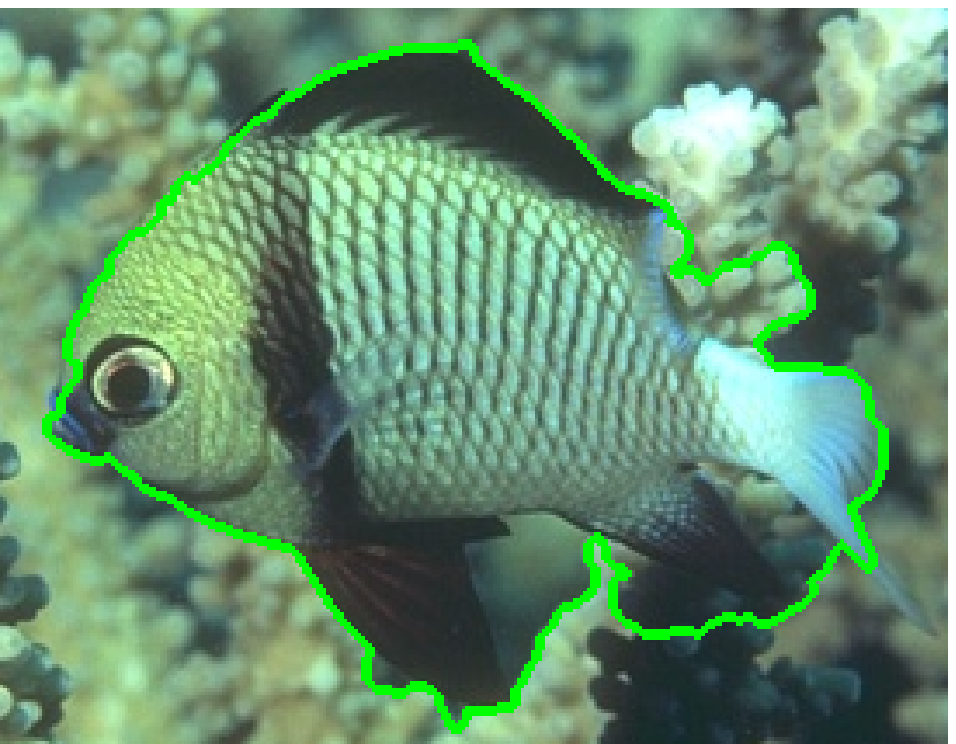}}
\end{tabular}
\end{center}
\caption {(a) : Seed images belonging to the first (top) and second (bottom) set
  of seeds provided by the GrabCut database. (b) Segmentations by Graph Cuts
  computed with seeds from the two sets, (c) Segmentations by CCMF.}
\end{figure}

\begin{table}
\label{DB_results_d}
\begin{center}
\begin{tabular}{c|c|c|c|c|c|}
\cline{2-6}
     & Dice Coeff. & GC & AT-CMF & CP-TV & CCMF \\
\cline{2-6}
\hline
\multicolumn{1}{|c|}{\multirow{3}{*}{First set of seeds}} & mean& 95.2 & 94.9 & 95.2  & 95.3 \\
\multicolumn{1}{|c|}{} & median & 96.2 & 96.1 &  96.3  & 96.3 \\  
\multicolumn{1}{|c|}{} & stand. dev.& 4.3 & 4.1 &  4.1  & 4.1 \\  
\hline
\hline
\multicolumn{1}{|c|}{\multirow{3}{*}{Second set of seeds}} & mean&
89.3 & 89.7 &  89.1  & 89.5 \\  
\multicolumn{1}{|c|}{} & median  & 91.2 & 92.0 &  91.7  &92.3  \\  
\multicolumn{1}{|c|}{} & stand. dev. & 8.4 & 7.7 &  8.5  & 8.6   \\  
\hline
\end{tabular}
\end{center}
\caption {Dice coefficient (percentage) computed
    between the segmentation mask and the ground truth image (provided
    by GrabCut database). The lines above the double bar show results
    for the first set of seeds, and the lines below the double bar
    show results obtained with the second set of seeds.}
\end{table}

Figure~\ref{DB_results_d} displays the performance results for these
algorithms. The Dice coefficient is a similarity measure between sets
(segmentation and ground truth), ranging from 0 to 1.0 for no
match and a perfect match respectively. All the tested algorithms show
very good results, with a Dice coefficient of 0.95--0.97 for the well
positioned seeds, and 0.89--0.92 for the second set of seeds, further away from
the objects. The CCMF and AT-CMF results are really close, and the mean
  is better than the GC  and the CP-TV results. 


\subsection{Extensions}
 
The primary focus of our experiments were to demonstrate that our CCMF
algorithm achieves a solution which does not exhibit metrication artifacts,
that it is fast with provable convergence and that the segmentation
quality is high.  In the remainder of this section, we show that the
algorithm can also incorporate unary terms and be equally applied in
3D. In fact, since CCMF is defined for an arbitrary graph, it could
even be used to perform clustering in this more generalized framework.
However, the benefit of avoiding gridding artifacts is less meaningful
when performing clustering on an arbitrary graph, and therefore we omit
any examples of this nature.

\subsubsection{Unary terms}

A simple modification of the transport graph $G$ permits the use of
unary terms to automatically specify objects to be segmented.  This
approach is similar to the use of unary terms in the graph cuts
computation \cite{Boykov2001Interactive}. The placement of unary terms
for adding data attachment constraints in the max-flow problem is
performed by adding edges linking every node of the lattice to the
source and to the sink. In the case of the CCMF problem, as weights
are defined on the nodes, we need to add intermediary nodes between
the grid and source on the one hand, and the grid and sink on the
other. However, since these intermediary nodes have just two edges
incident on them, these node weights are equivalent to edge weights,
meaning that our construction is equivalent to the use of unary terms
for AT-CMF that was pursued by Unger {\it et al.} \cite{ungerReport}. In our
case, we used weighted intermediary nodes simply to keep the CCMF
framework consistent. Considering that the original lattice is
composed of $n$ nodes, we add for each node an ``upper'' node linked
to the source and a ``lower'' node linked to the sink. An example of
this construction is shown in Fig.~\ref{fig_graph}.

\begin{figure}[ht]
\begin{center}
\begin{minipage}{0.32\textwidth}
\begin{center}
\includegraphics[width=0.6 \linewidth]{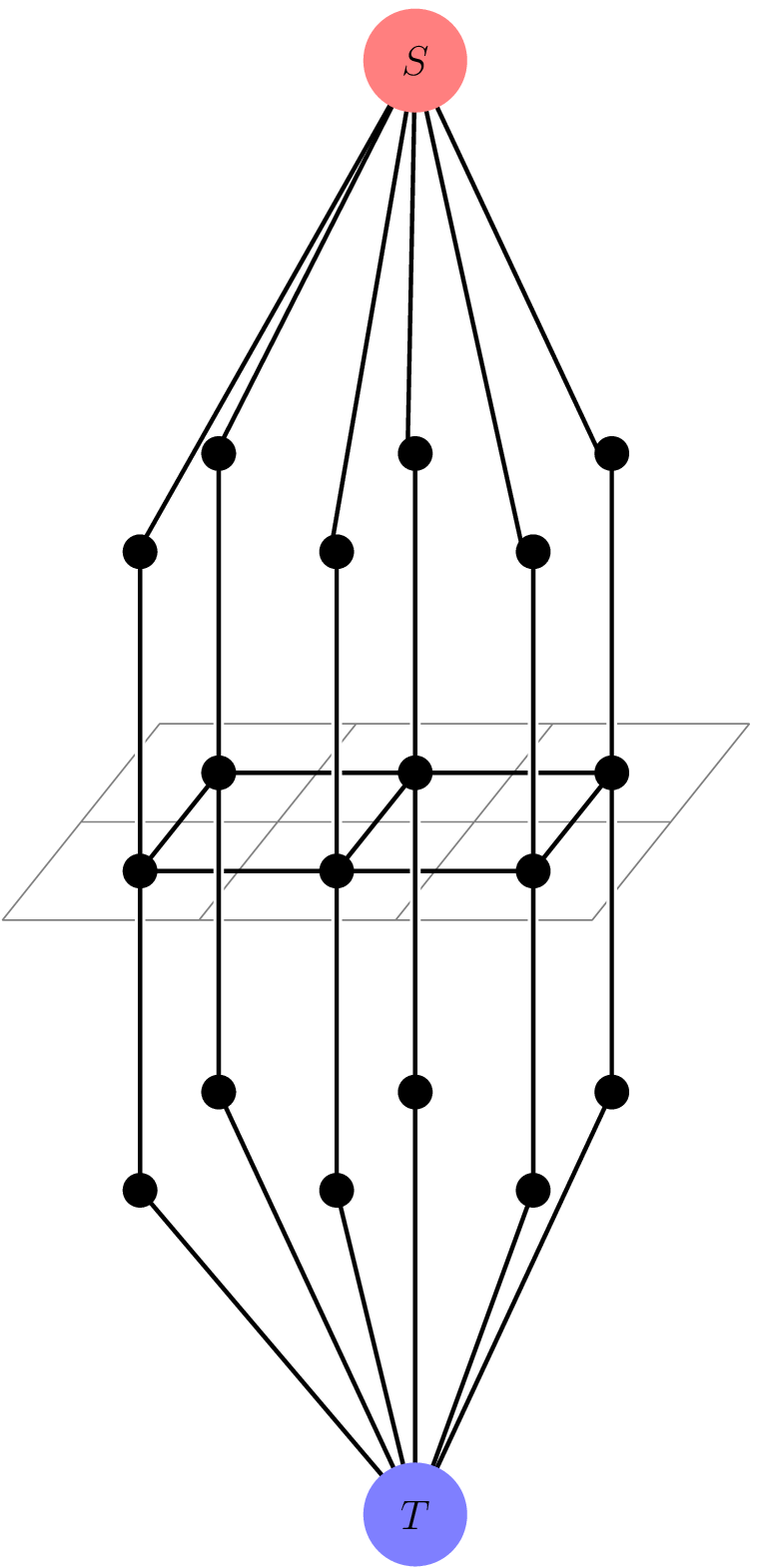}
\end{center}
\end{minipage}
\begin{minipage}{0.32\textwidth}
\begin{center}
\includegraphics[scale=0.615]{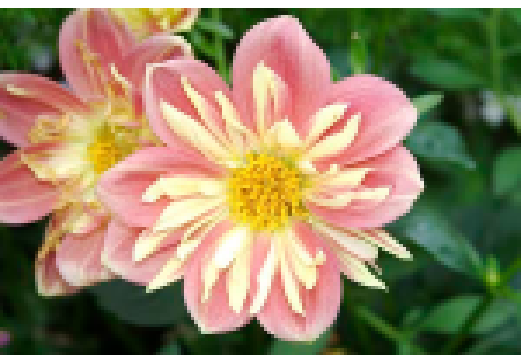}\\~\\
\includegraphics[scale=0.615]{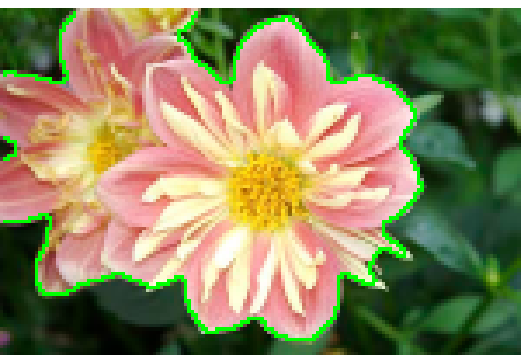}
\end{center}
\end{minipage}
\begin{minipage}{0.32\textwidth}
\begin{center}
\includegraphics[scale=0.5]{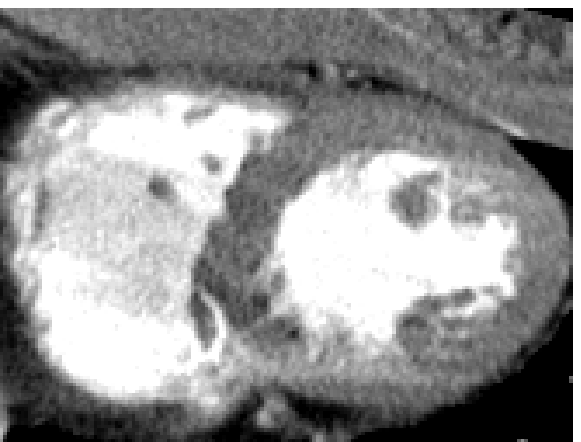}\\~\\
\includegraphics[scale=0.5]{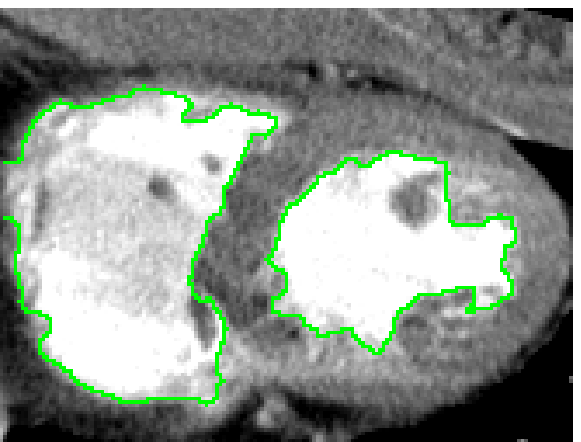}
\end{center}
\end{minipage}
\end{center}
\caption {Example of graph construction to incorporate unary terms for
  unsupervised segmentation. Each node is connected to a source node
  and sink node through an intermediary node which is weighted to
  reflect the strength of the positive and negative unary term.  This
  construction is then applied with a simple appearance prior to
  automatically segment two images.}
\label{fig_graph}
\end{figure}

The weights of the additional nodes may be set to reflect the node
priors for a particular application. For image segmentation, given
mean background and foreground color $BC$ and $FC$, we can set the
capacities of the background nodes to $g_i = \exp(-\beta
(BC_i-I_i)^2)$ and foreground nodes to $g_i = \exp(-\beta
(FC_i-I_i)^2)$.  Examples of results using these appearance priors are
shown on Fig.~\ref{fig_graph}.

\subsubsection{{\change Classification}}
  
\begin{figure}[bt]
\begin{center}
\begin{tabular}{cc}
\subfigure[ Network and true membership after split]{\includegraphics[width=0.45 \linewidth]{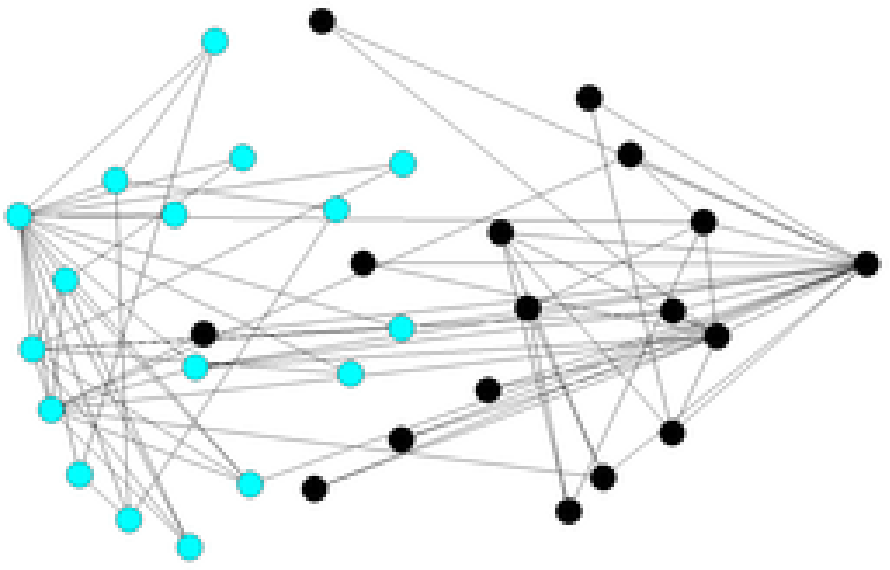}}&
\subfigure[Classification using CCMF]{\includegraphics[width=0.45 \linewidth]{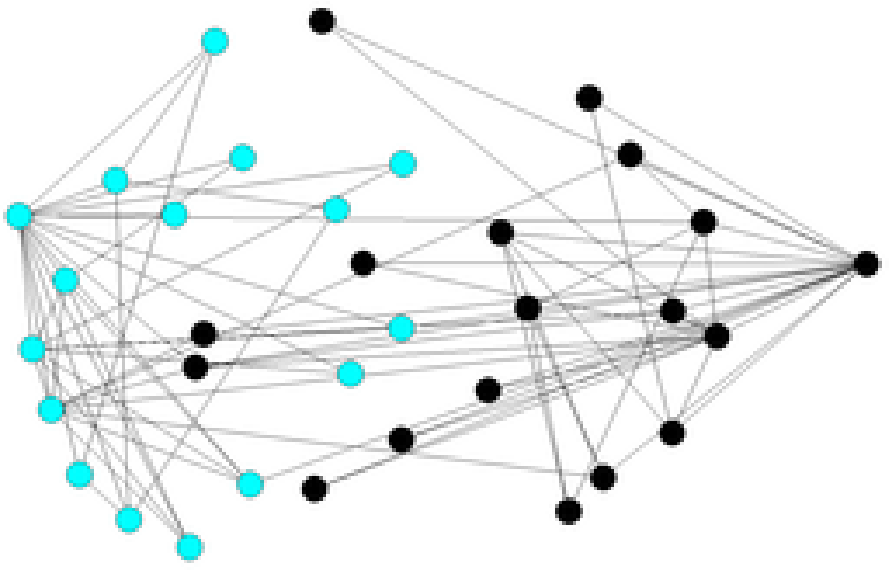}}
\end{tabular}
\end{center}
\caption {{\change Zachary's karate club network. The two leaders
    in conflict are represented by the left hand side and right hand
    side nodes. Just one node is misclassified with CCMF.  However,
    this node is known classically to be an unusual situation within
    the social network (see \cite{zachary}) which is not captured by any known
    algorithm.}}
\label{fig:classif} 
\end{figure}

{\change As the general CCMF formulation is defined on arbitrary
  graphs, it is possible to employ it in tasks beyond image
  segmentation. For instance CCMF may be employed to solve general
  problems of graph clustering, even on networks with no meaningful
  embedding such as social networks. We show a possible application in
  a classification example, considering the social network studied by
  Zachary \cite{zachary}. After the split of a university karate club
  due to a conflict between its two leaders, Zachary's goal was to
  study if it was possible to predict the side joined by the members,
  based only on the social structure of the club.  Classically, the
  graph is built by associating each member to a node, and edges link
  two members when special affinities are known between them.  As
  weights are associated here to the strength of coupling between members, different
  strategies are possible for assigning the weights onto the nodes,
  which is necessary for using CCMF. For example, a given node/member
  may be assigned by the mean of affinity with its neighboring members
  in the graph. This weighting strategy works well in this example
  (See Fig.~\ref{fig:classif}).  Another strategy would be to compute the
  solution in the dual graph, that is in the graph where nodes are
  replaced by edges and (weighted) edges are replaced by (weighted)
  nodes. We have shown in an example the ability of CCMF for
  classification, a task that can not be treated with finite
  element-based methods such as AT-CMF. Evaluating the performances of
  CCMF for classification may be a topic for future research.  }

\subsubsection{3D segmentation}

\begin{figure}[hb] 
\begin{center}
\begin{multicols}{2}{\includegraphics[width=0.81 \linewidth]{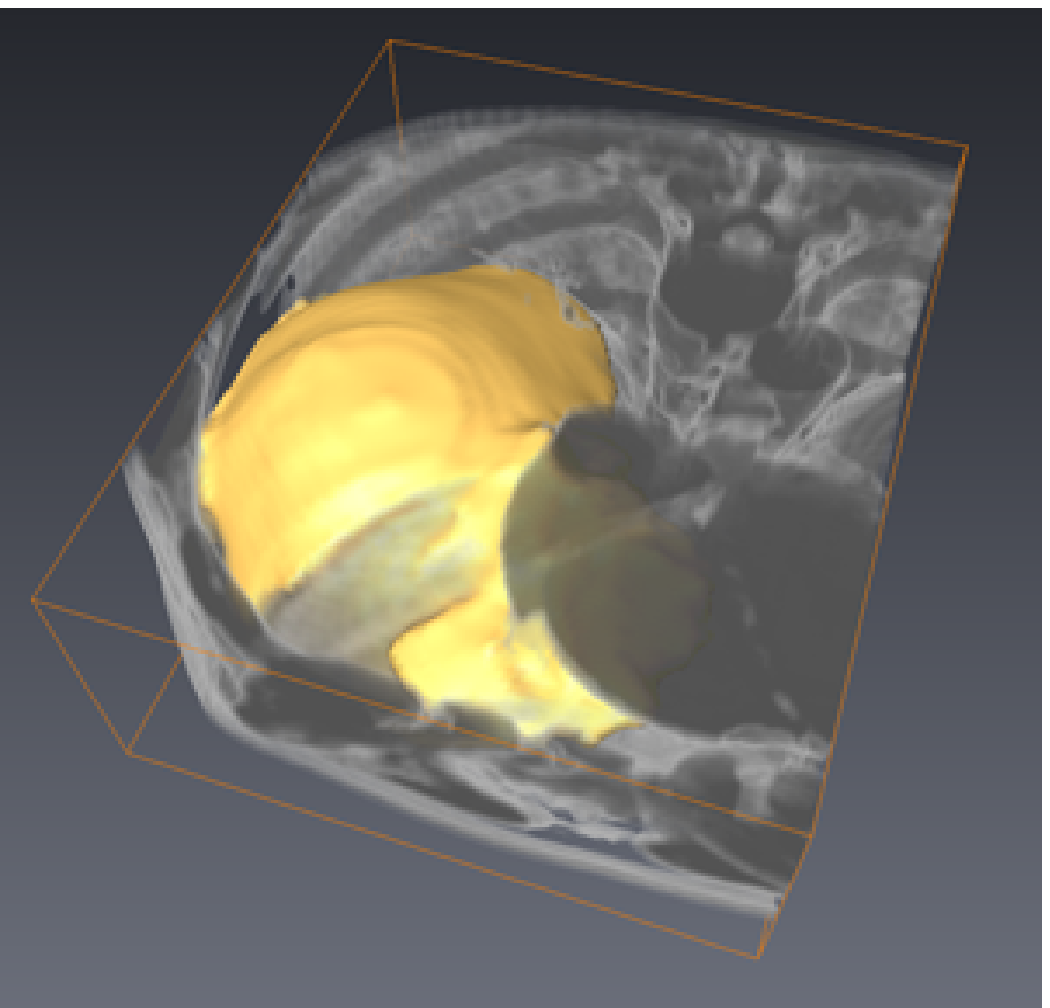}\\
 \begin{tabular}{ccc} 
 \includegraphics[width=0.3
  \linewidth]{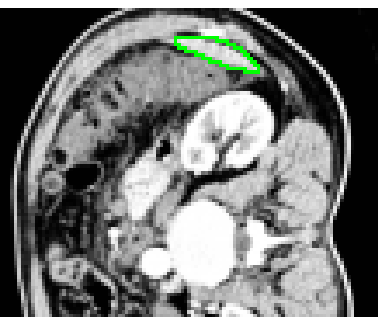}~~\includegraphics[width=0.3
  \linewidth]{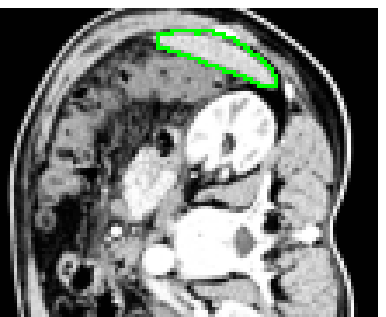}~~\includegraphics[width=0.3
  \linewidth]{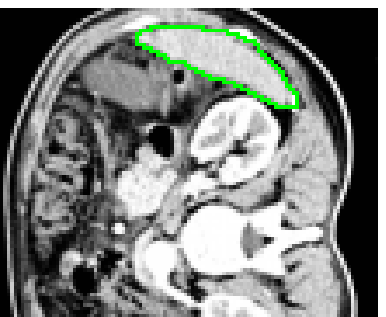}\\
\includegraphics[width=0.3
  \linewidth]{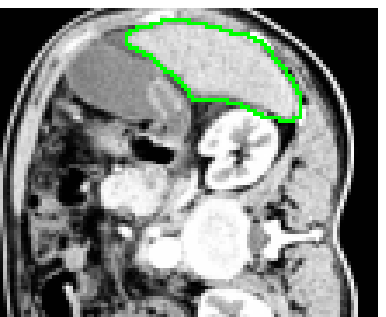}~~\includegraphics[width=0.3
  \linewidth]{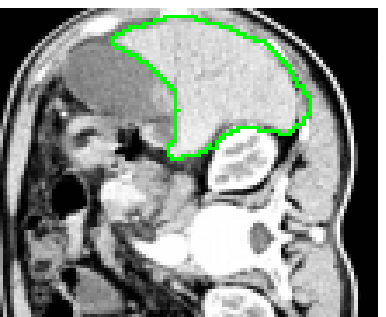}~~\includegraphics[width=0.3
  \linewidth]{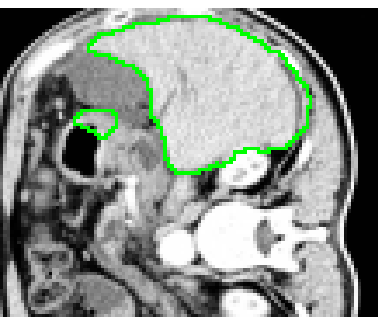}\\
\includegraphics[width=0.3
  \linewidth]{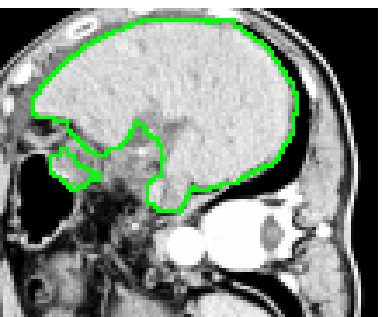}~~\includegraphics[width=0.3
  \linewidth]{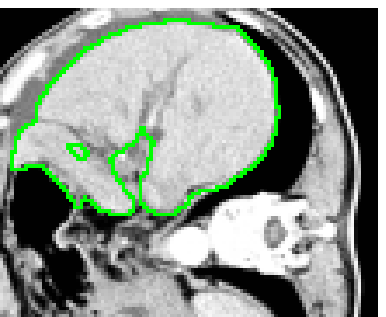}~~\includegraphics[width=0.3
  \linewidth]{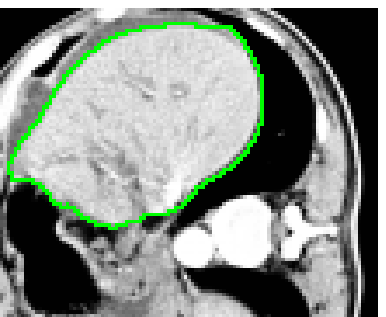}
\end{tabular}
}
\end{multicols}
\end{center} 
\caption{Liver segmentation in 3D by CCMF.}
\label{fig_liver}
\end{figure}

For 3D image segmentation, the minimal surface properties of CCMF
generate good quality results, as shown in Fig~\ref{fig_liver}. The
CCMF formulation applies equally well in 2D or 3D, since CCMF is
formulated on an arbitrary graph (which may be a 2D lattice, a 3D
lattice or an even more abstract graph).
In 3D, our CCMF implementation is suffering from memory limitations in the
direct solver we used, limiting its performances. Future work will
address this issue using a dedicated solver.




\section{Conclusion}

In this paper, we have presented a new combinatorial formulation of
the continuous maximum flow problem and a solution using an interior
point convex primal-dual algorithm. This formulation allows us to
optimize the maximum flow problem as well as its dual for
  which we provide an interpretation as a minimal surface
problem. This new combinatorial expression of continuous max-flow
avoids blockiness artifacts compared to graph cuts. Furthermore, the
formulation of CMF on a graph reveals that it is actually the fact
that the capacity constraints are applied to point-wise norms
  of flow {\emph{through nodes}} that allows us to avoid metrication
errors, as opposed to the conventional graph cut capacity constraint
{\emph{through individual edges}}. Additionally, it was
shown that our CCMF formulation provides a better approximation to the
analytic catenoid than the conventional AT-CMF discretization of the
continuous max-flow problem. {\change Contrary to graph cuts, we
  showed that CCMF estimates the
  true boundary length of objects. Finally, unlike finite-element based methods such as AT-CMF, the CCMF
  formulation is expressed for arbitrary graphs, and thus
  can be employed in a large variety of tasks such as classification.}

We provide in this paper an exact analytic expression of the dual problem,
convergence of the algorithm being guaranteed by the convexity of the problem.
In terms of speed, when an approximate solution is sufficient, our
implementation of CCMF in Matlab is competitive to the Appleton-Talbot approach,
which uses a system of PDEs, and a {\tt C++} implementation. The Appleton-Talbot
algorithm has the significant drawback of not providing a criterion for
convergence. In practice, this translates to long computation times when
convergence is difficult for the AT-CMF algorithm. In contrast, our
algorithm in this case is much faster, as we observe that the number of
iterations required for convergence does not vary much for CCMF (less than a
factor of two). The CCMF algorithm is simple to implement, and may be applied to
arbitrary graphs. Furthermore, it is straightforward to add unary terms to
perform unsupervised segmentation.

We have also compared our algorithm to known weighted combinatorial TV
minimization methods (CTV optimized with split-Bregman, CP-TV). We
have shown that our results are generally better for segmentation than
combinatorial TV-based methods, and that our implementation on CPU is
much faster and more predictable. Specifically, the number of
iterations required for convergence does not vary much for CCMF,
whereas it can vary of more than a factor of ten for CP-TV.  The deep
  study of the relationships between CCMF and combinatorial TV reveals
  that, in contrast with expectations from their duality 
    under restrictive assumptions in the continuous domain, their
  duality relationship is only weak in the combinatorial setting.
  In fact max-flow and total variation are different problems
    with different constraints, yielding different algorithms and
    different results. One key difference between max-flow and total
    variation is that max-flow algorithms were developed as
    segmentation algorithms, following the graph cuts framework. One
    consequence is that max-flow formulations have a null divergence
    objective, which is not present in total variation
    formulations. Null divergence can be viewed as a consequence or a
    necessity in order to obtain constant partitions almost
    everywhere, in particular binary ones. This difference is
    important because in the proposed CCMF framework we {\em impose} a
    tight null divergence constraint, which to our knowledge is novel
    for an isotropic formulation (graph-cuts have always had this
    constraint). AT-CMF and similar frameworks achieve null divergence
    if and when convergence is achieved. There is no null divergence
    constraint at all in total variation frameworks. As a consequence,
    while in practice it is possible to compare total variation and
    max flow formulations, they should be treated differently.
  However, the strong computational performance and
    segmentation quality results of CCMF as compared to TV (using
   the strongest known optimizations methods) suggests that
  it may be advantageous to apply our CCMF formulation to problems for
  which TV has proven effective (such as filtering).

Several further optimizations of CCMFs are possible, for instance
using multi-grid implementations, the possibility to use GPU to solve
the iterative linear system, the use of a dedicated solver for the
particular sparse linear system involved in the computation. 
 We also plan to compare the efficiency obtained with the interior point
 method to a first order algorithm for solving the CCMF problem.  

Ultimately, we hope to employ combinatorial continuous maximum flow as
a powerful segmentation algorithm which avoids metrication artifacts and
provides a fast solution with provable convergence.  Furthermore, we
intend to explore the potential of CCMF to optimize other energy
functions for which graph cuts or total variation have proved useful,
such as surface reconstruction or efficient convex filtering.


\bibliographystyle{splncs}
\bibliography{siam2010_CCMF}

\end{document}